\documentclass{article}

    \PassOptionsToPackage{numbers, compress}{natbib}

\usepackage[preprint]{neurips_2022}

\usepackage[utf8]{inputenc} 
\usepackage[T1]{fontenc}    
\usepackage{hyperref}       
\usepackage{url}            
\usepackage{booktabs}       
\usepackage{amsfonts}       
\usepackage{nicefrac}       
\usepackage{microtype}      
\usepackage{xcolor}         

\usepackage{enumitem}

\usepackage{etoolbox}
\newcommand{\zerodisplayskips}{%
  \setlength{\abovedisplayskip}{0pt}%
  \setlength{\belowdisplayskip}{0pt}%
  \setlength{\abovedisplayshortskip}{0pt}%
  \setlength{\belowdisplayshortskip}{0pt}}
\appto{\normalsize}{\zerodisplayskips}
\appto{\small}{\zerodisplayskips}
\appto{\footnotesize}{\zerodisplayskips}

\usepackage{amsmath, amssymb}
\usepackage{hyperref}
\usepackage{amsthm}
\usepackage{subcaption}

\usepackage[]{graphicx}
\usepackage{todonotes}
\usepackage{url}
\usepackage{mdframed}

\usepackage[export]{adjustbox}
\usepackage{algorithm}
\usepackage{algorithmic}
\newtheorem{remark}{Remark}
\newtheorem{definition}{Definition}

\newtheorem{lemma*}{Lemma}
\newtheorem{theorem*}{Theorem}
\newtheorem{theorem}{Theorem}
\newtheorem{lemma}{Lemma}
\newtheorem{corollary}{Corollary}

\newif\ifboldnumber

\usepackage{amsmath}
\usepackage{xparse}

\ExplSyntaxOn
\NewDocumentCommand{\RN}{m}
 {
  \textup{ \int_to_Roman:n { #1 } }
 }
\ExplSyntaxOff

\title{On the Algorithmic Stability and Generalization of Adaptive Optimization Methods}

\author{%
  Han Nguyen \\ 
  Department of Mathematical Sciences\\
  Carnegie Mellon University\\
  \texttt{hann1@andrew.cmu.edu} \\
   \And
   Hai Pham \\
   Language Technologies Institute \\
    Carnegie Mellon University\\
   \texttt{htpham@cs.cmu.edu} \\
   \AND
   {Sashank J. Reddi} \\
   Google Research \\
  New York \\
   \texttt{sashank@google.com} \\
   \And
   {Barnab\'{a}s P\'{o}czos} \\
   Machine Learning Departmant \\
    Carnegie Mellon University\\
   \texttt{bapoczos@cs.cmu.edu} \\

}

\begin{document}

\maketitle
\begin{abstract}
Despite their popularity in deep learning and machine learning in general, the theoretical properties of adaptive optimizers such as Adagrad, RMSProp, Adam or AdamW are not yet fully understood.
In this paper, we develop a novel framework to study the stability and generalization of these optimization methods. Based on this framework, we show provable guarantees about such properties that depend heavily on a single parameter $\beta_2$. Our empirical experiments support our claims and provide practical insights into the stability and generalization properties of adaptive optimization methods. 
\end{abstract}

\vspace{-3mm}
\section{Introduction}
\vspace{-2mm}
Recent years have witnessed a surge of interest in adaptive optimization methods for deep learning settings.
For instance, Adam~\citep{kingma2014adam} --- despite its drawbacks in theory~\citep{reddi2019convergence} --- is often used in practice. Other popular choices of adaptive methods include Adagrad~\citep{duchi2011adaptive}, 
RMSProp~\citep{rmsprop2012} and AdamW~\citep{loshchilov2019decoupled}. All of those methods are now the standard choices in deep learning communities. 
However, such adaptive methods are not always the winner, because people still find the traditional methods of Gradient Descent (GD) and Stochastic Gradient Descent (SGD) more superior in certain cases,  both in terms of optimization and generalization~\citep{reddi2019convergence,wilson2017marginal}.  
Despite their different yet prevalent usages, the choices of which optimizer to use are not entirely clear and often come from practice or user experiences. Likewise, there has been no known theoretical foundation that thoroughly explains the reasons as to why we use a particular optimizer in this case, but a different one in another case. In this work, we build such a theoretical framework that offers reasons concerning that important question.

Retrospectively, the development of adaptive optimization methods could be dated back to the work of \cite{mcmahan2010adaptive} and \cite{duchi2011adaptive} in the online learning setting.  Their methods work well in sparse settings~\citep{duchi2013estimation}. But the adaptive learning rate often decays rapidly for dense gradients, thus slowing down the convergence speed of their methods significantly. To tackle this issue, some approaches such as Adadelta~\citep{zeiler2012adadelta}, RMSProp~\citep{rmsprop2012}, and Adam~\citep{kingma2014adam} proposed to use an exponential moving average of past squared gradients. In particular, Adam has become an increasingly popular optimizer in deep learning due to its effectiveness in the early training stage~\citep{tong2019calibrating}. However, few facts are known about these adaptive methods in terms of optimization and generalization.

Generally, when it comes to a learning problem, there are two most important metrics that people care about, which are directly related to the optimizers being used. 
One is the convergence property such as convergence guarantees and convergence rate, which offer insights about the fundamental quality of the optimizer in terms of reliability and run-time complexity~\cite{reddi2019convergence}. The other metric is generalization, which evaluates how well the optimizer does on unseen test data comparing to its performance on the train data ~\citep{morcos2018importance}. Oftentimes there is a tradeoff between those two metrics: one can not usually perform well on train data (i.e. convergence) as well as on test data (i.e. generalization)~\citep{NIPS2007_0d3180d6}. Specifically, the generalization error is usually large when the training loss is small and vice versa. Importantly, the generalization error is upper bounded by the stability of the optimizer~\citep{bousquet2002stability}. Roughly speaking, stability measure the sensitivity of an optimizer with respect to the change of train data. This is an important concept because we do not want to use an optimizer that finds very different solutions when we just slightly change the train data. In this paper, we study the adaptive optimization methods through the lens of stability and draw some conclusions about their generalization as a consequence. 

\textbf{Our contributions.} First, we provide a novel analysis for the stability of adaptive optimization methods which is different from the previous analysis for the stability of SGD and SGD with momentum by \citep{hardt2016train, ramezanikebrya2018stability}. 
Second, we show that the stability of the adaptive optimization methods depends on the parameter $\beta_2$. Specifically, the bigger the $\beta_2$, the more stable the algorithm is. Our experiments confirm this theory.
Third, we show that when $\beta_2$ goes to 1 at the rate $1-\alpha/t$, the stability grows as $O(T^{r}), r < 1$ where $T$ is  number of iterations. Our experiments reflect this rate. 
Finally, we show that weight decay helps improve the stability and generalization of Adam. Our experimental results also validate this theory.
\vspace*{-\baselineskip}
\section{Related Work}
\vspace*{-\baselineskip}

%
\textbf{Convergence.} The convergence of adaptive  optimization methods has been studied under various conditions. 
Adaptive learning rate in Adagrad was shown to improve convergence in sparse, convex setting~\citep{duchi2011adaptive,mcmahan2010adaptive} and in non-convex setting~\citep{ward2019adagrad}.
In convex setting, \cite{reddi2019convergence} proposed AMSGrad to fix a convergence issue in Adam in certain online and stochastic learning cases due to its constant learning rate, and hence stimulated following work trying to improve Adam. 
\cite{luo2019adaptive} proposed Adabound that used a clipping technique and gradually forced adaptive learning rate to converge to a predefined value. 
\cite{tong2019calibrating} proposed SAdam and SAMSGrad, which applied softmax to each coordinate of the adaptive learning rates to keep them low variance. 
In non-convex settings, the convergence of Adam 
has been studied by several works. \cite{zaheer2018adaptive} showed that Adam converged
when increasing the batch size. 
\cite{de2018convergence} then showed that under some mild assumptions, the deterministic Adam converged. 
Additionally, \cite{zhou2020convergence,chen2020closing} proved the convergence of generalize versions of Adam. 
Recently, some works have proven that AMSGrad converged in the weakly convex setting \citep{nazari2020adaptive, alacaoglu2020convergence}. 
While all of them showed that Adam converged at the rate of $O(\frac{1}{\sqrt{T}})$, none of them answered the question as to why Adam-like algorithms have any benefits over SGD in terms of optimization, unlike in our work. 

\textbf{Stability.} Algorithmic stability is a fundamental concept in learning theory that dates back to the pioneering work by \cite{Rogers1978AFS}, who showed that the expected sensitivity of a classification algorithm (kNN in particular) to changes of individual examples could be used to obtain a variance bound of a leave-one-out estimator. Later, \cite{bousquet2002stability} introduced the concept of uniform stability, from which the empirical risk minimization (EMR) was uniformly stable if the objective function was strongly convex. This concept was further extended to study randomized algorithms such as bagging and boosting \citep{JMLR:v6:elisseeff05a}. \cite{ShalevShwartz2010LearnabilitySA} then proposed weaker on-average stability and used it to study unregularized learning algorithms \citep{gonen2017fast}. Very recently, almost optimal high-probability bounds were established for uniformly stable algorithms by developing elegant concentration inequalities for weakly-dependent random variables \citep{pmlr-v65-maurer17a, feldman2019high, bousquet2020sharper}. Other notion of stability were also proposed such as  the uniform argument stability \citep{liu2017algorithmic}, hypothesis stability \citep{bousquet2002stability} and hypothesis set stability \citep{foster2020hypothesis}. The deep connection between algorithmic stability and learnability was identified \citep{ShalevShwartz2010LearnabilitySA,Rakhlin05stabilityresults}. Our work is built upon the foundation of those ideas.

\textbf{Generalization.} 
In a seminal paper, \cite{hardt2016train} used uniform stability to derive generalization bound for SGD in the strongly convex, convex, and non-convex settings. This stability analysis was further refined by~\cite{zhou2019generalization} by using the on-average variance and was more recently extended to the non-smooth setting by~\cite{lei2020finegrained}. In another work, \cite{kuzborskij2018datadependent} developed a bound for SGD concerning data-dependent stability, a nice property that showed how initialization would affect generalization.  
Furthermore, \cite{chen2018stability} developed a tradeoff lower bound between stability and convergence of iterative optimization algorithms. They also proved stability bound for momentum optimizers such as Nesterov acceleration and a heavy-ball method with a fixed momentum when the objective function is quadratic. \cite{ramezanikebrya2018stability, NIPS2017_ad71c82b} derived the stability bound for heavy-ball method for general convex and non-convex functions. In contrast to SGD and SGD with momentum, there have been no existing results in the literature on the stability and generalization for adaptive optimization methods, which are addressed by us in this work. 

Other approaches
were also developed for characterizing the generalization error as well as the estimation error, which are orthogonal to our work. Some of them were based on the information-theoretic approach \citep{pmlr-v51-russo16,NIPS2017_ad71c82b, negrea2020informationtheoretic, jose2021informationtheoretic}, the algorithm robustness framework \citep{xu2010robustness, zahavy2017ensemble}, large-margin theory \citep{bartlett2017spectrallynormalized}, and the classical VC theory \citep{788640}. 

\vspace{-2mm}
\section{Preliminaries}
\vspace{-2mm}
In order to establish our main results, we first formulate related fundamental concepts concerning generalization, algorithmic stability, and adaptive optimization methods.
\vspace{-2mm}
\subsection{Excess risk decomposition} \label{sec:decomposition}
\vspace{-2mm}
In this section, we briefly review fundamental concepts of empirical risk minimization and excess risk decomposition. We consider the standard setting in supervised learning problems. In this setting, we are given a sample $S = \{z_1,\ldots,z_n\}$ of size $n$ where  each data point $z_i$ lies in some space $\mathcal{Z}$ and is drawn i.i.d according to an unknown distribution $\mathbb{P}\in\mathcal{P}$. Given a loss function $f(x;z)$, we want  to find some $x\in \Omega\subset \mathbb{R}^d$ that minimize the expected loss $ F(x) = \mathbb{E}_{z \sim \mathbb{P}}f(x;z) = \int_{\mathcal{Z}}f(x,z)d\mathbb{P}(z).    $

Since the distribution $\mathbb{P}$ is unknown, we instead minimize the empirical loss $F_S(x) = \frac{1}{n}\sum_{i=1}^nf(x;z_i)$.
We denote by $x_S$ an estimator computed from sample $S$. The statistical question is how to bound the excess risk in terms of the difference between the population risk evaluated at $x_S$ and the true minimal risk over the entire parameter space $\Omega$,
  $\mathcal{R}(x_S) = F(x_S) - \inf_{x\in\Omega}F(x)$.  
We have the following lemma about expected risk decomposition. 

\begin{lemma}\label{lem:decomposition}
Let  $x^*_S$ be an empirical risk minimizer. We then have that 
$$
 \mathbb{E}_S[\mathcal{R}(x_S)]\nonumber 
\le  
{\mathcal{E}_{gen}}
+ 
{\mathcal{E}_{opt}} = \mathbb{E}_S[F(x_S) - F_{S}(x_S)] + \mathbb{E}_S[F_S(x_S) - F_S(x_S^*)],
$$
where $\mathcal{E}_{gen}$ and $\mathcal{E}_{opt}$ are the expected generalization error and optimization error respectively.
\end{lemma}
\vspace*{-2mm}
For proof, please see Appendix~\ref{proof:decomposition}. 
Similar derivations to this result can also be found in \cite{Bousquet2004, NIPS2007_0d3180d6, chen2018stability}.

\vspace{-2mm}
\subsection{Algorithmic stability}
\vspace{-2mm}
It turns out that the expected generalization error can be controlled by various notions of algorithmic stability.
For a thorough review, please refer to \cite{bousquet2002stability, ShalevShwartz2010LearnabilitySA}.  For the purpose of this paper, we are only interested in the notion of uniform  stability  introduced by~\cite{bousquet2002stability}.
\begin{definition}
An algorithm, which output a model $x_S$ for sample $S$, is $\tau-uniformly\ stable$ if for all $k\in\{1,\ldots,n\}$, for all data sample pair $S = \{z_1,\ldots,z_k,\ldots,z_n\}$,  and $S' = \{z_1,\ldots,z'_k,\ldots,z_n\}$, where  $z_i$ and $z'_k$ are i.i.d sampled from $\mathbb{P}$, we have $sup_{z\in\mathcal{Z}}|f(x_{S};z) - f(x_{S'};z)|\le \tau. $
\label{definition1}
\end{definition}
This definition implies that when we randomly replace one arbitrary sample from the training data with another i.i.d one, the deviation of the loss between two output models is uniformly within some $\tau$. 
Moreover, the smaller the $\tau$ is, the more stable the algorithm is. One important property of uniform stability is that it implies generalization, as formulated in the following theorem. 
\begin{theorem} \label{theorem:generalization}
If an algorithm that outputs a model $x_S$ for sample $S$ is $\tau$-uniformly-stable, then its expected generalization error is bounded as follows,
\begin{align}
|\mathbb{E}_S[F(x_S) - F_S(x_S)]|  \le \tau.
\end{align}
\end{theorem}
The proof of Theorem \ref{theorem:generalization} can be found in \cite{hardt2016train}.
For the rest of this paper, we focus on finding an upper bound for $\tau$ in Definition \ref{definition1} when using the adaptive optimization methods. 
\begin{remark}
The concept of uniform stability has a strong connection with the sensitivity analysis in optimization. Sensitivity analysis studies  the sensitivity of the optimal value of an optimization problem with respect to perturbations of the problem's constraints. For more detail about sensitivity analysis, please refer to \cite{Bonnans1998OptimizationPW}.
\end{remark}

\vspace{-2mm}
\subsection{Adaptive optimization methods}
\vspace*{-\baselineskip}
\begin{minipage}{0.48\textwidth}
\begin{align}
m_t & = \alpha_t m_{t-1} + (1-\alpha_t)\nabla f(x_{t-1}, z_t) \label{eq:ada1},   \\
v_t &= \beta_t v_{t-1} +(1 - \beta_t)\nabla f(x_{t-1}, z_t)^2\label{eq:ada2}, \\
x_{t} &= x_{t-1} - \eta_t \frac{m_t}{\sqrt{v_t + \epsilon}}\label{eq:ada3},
\end{align}
\end{minipage}
\hfill
\begin{minipage}{0.49\textwidth}
\begin{align}
\dot{m}(t) & = p(t)(\nabla F(x(t)) - m(t)) \label{eq:ode1}, \\
\dot{v}(t) & = q(t)(\nabla F(x(t))^2 - v(t)) \label{eq:ode2}, \\
\dot{x}(t) & = -\frac{m(t)}{\sqrt{v(t)+\epsilon}} \label{eq:ode3},
\end{align}
\end{minipage}

In this section, we review the general formulation the adaptive optimization methods (AOM). Let $(\mathcal{Z}, \mathcal{F}, \mathbb{P})$ be a probability space. Consider the minimization problem $\min_{x\in\Omega} F(x)$, where $F(x)$ is defined in Section~\ref{sec:decomposition}, 
and $f : \Omega\times\mathcal{Z} \rightarrow \mathbb{R}$ is a measurable map. For a fixed $z$, the mapping $x \mapsto f(x,z)$ is supposed to be differentiable and its gradient w.r.t $x$ is $\nabla f(x,z)$. Starting from $(x_0,m_0,v_0)\in \mathbb{R}^d\times\mathbb{R^d}\times\mathbb{R}_+^d$, AOM generates updates at time step $t$ as shown in Equations (\ref{eq:ada1}, \ref{eq:ada2}, \ref{eq:ada3}), 
where $\alpha_t, \beta_t\in (0,1)$ and $\epsilon > 0$ (to prevent zero division). 
When $\alpha_t = \alpha, \beta_t = \beta$, the above formulation becomes Adam algorithm. Notice that the $\alpha$ and $\beta$ in this case corresponding to the $\beta_1$ and $\beta_2$ in  Adam optimizer. When $\alpha_t = 0, \beta_t = 1- \frac{\alpha}{t}$, we obtain Adagrad optimizer.

In the continuous setting, the formulation above can be shown as the noisy Euler's discretization of the ordinary differential equation (ODE) in Equations~(\ref{eq:ode1}, \ref{eq:ode2}, \ref{eq:ode3}), 
where $(x(0), m(0), v(0)) = (x_0, m_0, v_0)$~\citep{barakat2020stochastic}. 
This ODE's formulation helps us overcome the difficulties in the latter analysis. We will look at AOM as a system of two equation with variables $(x, v)$ instead of treating $v$ as a separate adaptive learning rate.
\vspace{-2mm}
\section{Stability of adaptive optimization methods}
\vspace{-2mm}
In this section, we study the stability of the AOM optimizer when $\alpha_t = 0$ for all $t$. We leave the analysis with $\alpha_t>0$ as a future research direction. In this case, we show that the stability of the AOM method depends on the parameter $\beta_t$. Moreover, we  show that the AOM method is uniformly stable when $\beta_t \le 1-\alpha/t$. As a consequence, the Adagrad algorithm is a uniformly stable algorithm.

For latter convenience, let us first set up some notations and assumptions. We denote  
%
    $S = \{z_1,\ldots,z_i,\ldots,z_n\}$ and 
    $S' = \{z_1,\ldots,z_i',\ldots,z_n\}$
to be two training sets which are drawn from the unknown distribution $\mathbb{P}$. All data points in $S$ and $S'$ are the same except the data point at the i-\textit{th} position. Given a fixed model,  we consider training that model by using the AOM on the above two  data sets. We denote

\begin{minipage}{0.48\textwidth}
\begin{align}
     v_{t+1} &= \beta_t v_t + (1-\beta_t)\nabla f(x_t,z_{i_t})^2
    \\
    x_{t+1} &= x_t - \eta_t \frac{\nabla f(x_t,z_{i_t})}{\sqrt{v_{t+1}+\epsilon}} 
\end{align}
\end{minipage}
\hfill
\begin{minipage}{0.49\textwidth}
\begin{align}
    v'_{t+1} &= \beta_t v'_t + (1-\beta_t)\nabla f(x'_t,z'_{i_t})^2\\
    x'_{t+1} &= x'_t - \eta_t \frac{\nabla f(x'_t,z'_{i_t})}{\sqrt{v'_{t+1}+\epsilon}}
\end{align}
\end{minipage}

where $(x_t,v_t)$ and $(x_t',v_t')$ are the outputs of the AOM when running on $S$ and $S'$ respectively.
Specifically, at every time step $t$, we pick an index $i_t$ uniformly random from the set $\{1,\ldots,n\}$ and update the parameters according to the AOM update rules. Notice that we use two different notation $z_{i_t}$ and $z'_{i_t}$ because we pick data from two different training sets $S$ and $S'$. In addition, for a fixed $t_0\in\{1,\ldots,n\}$,  we denote
\begin{align}
    \delta_t = \|x_t - x_t'\|_2&, \quad\Delta_t = \mathbb{E}[\delta_t|\delta_{t_0}=0],\\ 
     \sigma_t = \|v_t - v'_t\|_2&,\quad \Sigma_t = \mathbb{E}[\sigma_t|\delta_{t_0}=0].
\end{align}

The analysis relies on the following set of assumptions:
\begin{enumerate}[nosep,noitemsep,nolistsep,label=(\arabic*)]
    \item $f(.,z)$ is  $L$-Lipschitz and $\mu$-smooth for all $z\in \mathcal{Z}$, i.e.
    \begin{align}
    |f(x,z) - f(y,z)|\le \mu\|x-y\|_2, \text{and} \quad \quad 
    \|\nabla f(x,z) - \nabla f(y,z)\|_2 \le L\|x - y\|_2
    \end{align}
    which imply $\|\nabla f(x,z)\|_2 \le \mu$ for all $z\in \mathcal{Z}$.
    \item There exists $M >0$ such that $f(x,z)\le M$ for all $x,z\in \Omega\times \mathcal{Z}$.
    \item There exists $\lambda_1, \lambda_2 \ge 0$ such that 
\begin{align}
\min_{t}\min_{i}\{v_{t+1}, v'_{t+1}\} \ge \lambda_1, \quad \quad  \quad \quad 
\max_{t}\max_{i}\{v_{t+1}, v'_{t+1}\} \le \lambda_2.
\end{align}
\end{enumerate}
Note that assumption (1) is standard when analyzing the stability of optimization algorithms \citep{bousquet2002stability,hardt2016train}. Assumption (2) is easily achieved if we restrict our domain to a compact set. 
Assume that we run the AOM on $S$ and $S'$ for $T$ steps to get the outputs $x_T$ and $x'_T$. We first have the following lemma which controls the difference of the loss at $x_T$ and $x'_T$. The proof of this lemma can be found in \citep{hardt2016train}. For completeness, we also provide a proof of this lemma in the appendix.
\begin{lemma} \label{lem:stability} 
Assume that assumptions (1) and (2) are satisfied. Let $S$ and $S'$ be two samples of size $n$ differing in only a single example. Denote by $x_T$ and $x'_T$ the output after $T$ steps of   AOM on $S$ and $S'$,  respectively.   Then for all $z\in \mathcal{Z}$ and $t_0\in\{1,\ldots,n\}$, we have that 
\begin{align}
  \mathbb{E}|f(x_T;z) - f(x'_T;z)| &\le 2M\frac{t_0}{n} + \mu\mathbb{E}[\delta_T|\delta_{t_0} = 0] \nonumber = 2M\frac{t_0}{n} + \mu\Delta_T.
\end{align} \label{lem:one}
\end{lemma}
\vspace*{-\baselineskip}
The insight of Lemma~\ref{lem:one} is that AOM has to take several steps before it encounters the index $i$ where the two data sets are different. This will be important later in our analysis because our loss function could be non-convex. In real application, $n$ is usually very big so it will take a long time before AOM reach the index $i$ where the two data sets are different. From now on, we will fix $t_0\in\{1,\ldots,n\}$. Given this lemma, it remains  to estimate the quantity $\Delta_T$. Next, we have the following theorem.  
\begin{theorem} \label{theorem2}
Assume that assumptions (1), (2), (3) are satisfied. Suppose we run AOM with step size $\eta_t$ on $S$ and $S'$ respectively, starting from the same initialization. Then for any outputs $x_t$ and $x'_{t}$, we  have that 
\begin{align}
    \Delta_{t+1} \le \Delta_t  &+ \eta_t \left(1-\frac{1}{n}\right)\frac{L}{\sqrt{\lambda_1+\epsilon}}\left[1 + \frac{\mu^2 (1-\beta_t)}{\lambda_1+\epsilon}\right]\Delta_t\nonumber\\
    & + \eta_t \left(1-\frac{1}{n}\right)\frac{\mu\beta_t}{2\sqrt{\lambda_1+\epsilon}(\lambda_1+\epsilon)}\Sigma_t
    + \eta_t\frac{1}{n}\frac{2\mu}{\sqrt{\lambda_1+\epsilon}}.
\end{align}
\end{theorem}

For a full proof, please see Appendix \ref{appx:proof_theorem2}.
This theorem gives us a general upper bound of $\Delta_t$. However, the bound also involves the term $\Sigma_t$ which is different from that of SGD \citep{hardt2016train}. This is expected because the formulation of the AOM involves a system of equations. Thus, in order to bound $\Delta_t$, we also have to construct an upper bound for $\Sigma_t$, as described in Theorem~\ref{theorem3}.
\begin{theorem} \label{theorem3}
Assume that assumptions (1), (2), (3) are satisfied. Suppose we run AOM with step size $\eta_t$ on $S$ and $S'$ respectively, starting from the same initialization. Then for any output $x_t$ and $x_t'$ we have that 
\begin{align}
\Sigma_{t+1} \le \beta_t \Sigma_t & +  2(1- \beta_t)\left(1-\frac{1}{n}\right)L\mu\Delta_t + 2\frac{1}{n}(1-\beta_t)\mu^2. \end{align}
\end{theorem}
\vspace{-3mm}
The full proof is in Appendix~\ref{appx:proof_theorem3}. Then combining Theorems~\ref{theorem2} and \ref{theorem3}, we obtain the following dynamical inequality:  
\begin{align}
\begin{split}
\begin{bmatrix}
\Delta_{t+1}\\
\Sigma_{t+1}
\end{bmatrix}  & \le  
\begin{bmatrix}
\mathcal{U}_t & \mathcal{P}_t \\ 
\mathcal{Q}_t & \mathcal{R}_t 
\end{bmatrix} \begin{bmatrix}
\Delta_t\\
\Sigma_t
\end{bmatrix} + \frac{1}{n}\begin{bmatrix}
\eta_t\frac{2}{\sqrt{\lambda_1+\epsilon}}\mu\\
(1-\beta_t)2\mu^2  \\
\end{bmatrix},  
\end{split}
\end{align}
where 

\vspace*{-\baselineskip}
\begin{minipage}{0.48\textwidth}
\begin{align*}
\mathcal{U}_t &= 1 + \eta_t \left(1-\frac{1}{n}\right)\frac{L}{\sqrt{\lambda_1+\epsilon}}\left[1 + \frac{\mu^2 (1-\beta_t)}{\lambda_1+\epsilon}\right], \\
\mathcal{R}_t &= \beta_t,
\end{align*}
\end{minipage}
\hfill
\begin{minipage}{0.49\textwidth}
\begin{align*}
    \mathcal{P}_t &= \eta_t \left(1-\frac{1}{n}\right)\frac{\mu\beta_t}{2\sqrt{\lambda_1+\epsilon}(\lambda_1+\epsilon)},
\\
    \mathcal{Q}_t &= 2 L \mu (1- \beta_t)\left(1-\frac{1}{n}\right)
\end{align*}
\end{minipage}

\vspace*{-\baselineskip}
Denote 
$A_t =  
\begin{bmatrix}
\mathcal{U}_t & \mathcal{P}_t \\ 
\mathcal{Q}_t & \mathcal{R}_t 
\end{bmatrix}$ and
\begin{align*}
U_t &: = \left(1-\frac{1}{n}\right)\frac{L}{\sqrt{\lambda_1+\epsilon}}\left[1 + \frac{\mu^2 (1-\beta_t)}{\lambda_1+\epsilon}\right] 
\le  U:= \left(1-\frac{1}{n}\right)\frac{L}{\sqrt{\lambda_1+\epsilon}}\left[1 + \frac{\mu^2 }{\lambda_1+\epsilon}\right],\\   
V_t & : = \left(1-\frac{1}{n}\right)\frac{\mu\beta_t}{2\sqrt{\lambda_1+\epsilon}(\lambda_1+\epsilon)} 
\le 
V:=\left(1-\frac{1}{n}\right) \frac{\mu}{2\sqrt{\lambda_1+\epsilon}(\lambda_1+\epsilon)},
%
\\
W& : =  2\left(1-\frac{1}{n}\right)L\mu,\quad \quad\quad\quad\quad\quad
Y:= \frac{2\mu}{\sqrt{\lambda_1+\epsilon}},\quad \quad\quad\quad\quad\quad Z:= 2\mu^2.    
\end{align*}
We then can rewrite 
$A_t = \begin{bmatrix}
1+ \eta_t U_t & \eta_t V_t \\
(1-\beta_t)W & \beta_t
\end{bmatrix}$ and 
$
\begin{bmatrix}
\Delta_{t+1}\\
\Sigma_{t+1}
\end{bmatrix}\le  A_t\begin{bmatrix}
\Delta_t\\
\Sigma_t
\end{bmatrix} + \frac{1}{n}\begin{bmatrix} \eta_t Y\\ (1-\beta_t)Z\end{bmatrix}\label{maineq}.
$
Thus the stability of AOM now depends on the norm of the matrix $A_t$. The following lemma gives an upper bound on the operator norm of $A_t$ for all $t\in\mathbb{N}$.
\begin{lemma} \label{normest}
Let $\eta_t = \frac{c}{t}$ and $\beta_t = 1-\alpha_t , \alpha_t\in (0,1)$. We then have that 
\begin{align}\|A_t\|_2  \le exp\left(\frac{c}{t}\sqrt{D_1} + \frac{c^2}{t^2}\frac{D_1}{2} + \alpha_t D_2 \right) \end{align}
where $D_1 = U^2 + V^2, D_2 = \sqrt{W^2+1} + \frac{1}{2}(W^2 + 1)$.
\end{lemma}


The proof of this lemma is in Appendix~\ref{appx:proof_lemma3}. Armed with these results, we are now ready to give a quantitative estimate of the deviation of the model's parameters when running AOM on two data sets $S$ and $S'$.

\begin{theorem} \label{theorem4}
Assume that  assumptions (1), (2), (3) are satisfied.  Starting from the same initialization, suppose that we run AOM with step size $\eta_t$ on $S$ and $S'$ respectively for $T$ iterations and output $x_T, x'_T$. Let $\eta_t = \frac{c}{t}$ and $\beta_t = 1-\alpha_t , \alpha_t\in (0,1)$. We then have that 
\begin{align}
    {\|\Delta_T\|_2} &{\le  \frac{1}{n} exp\left(\gamma \frac{c^2D_1}{2}\right) \times}  
    {\sum_{t=t_0+1}^Texp\left( D_2\sum_{k= t+1}^T\alpha_t\right)\left(\frac{T}{t}\right)^{c\sqrt{D_1}}\left(\frac{1}{t}cY+ \alpha_t Z\right)}
\end{align}
where we define $\gamma = \sum_{k=1}^\infty \frac{1}{k^2}$.
\end{theorem}

For a full proof, please see appendix \ref{appx:proof_theorem4}. Theorem~\ref{theorem4} and Lemma~\ref{lem:one} provide  a general bound on the stability of the AOM.
If  we replace $\{\beta_t\}$ by specific values, we obtain the stability bound for some well-know instances of AOM. 
\begin{corollary}(Stability bound for Adam with $\beta_1 = 0$) \label{coro}
Let $\eta_t = \frac{c}{t}$ and $\beta_t = \beta$. We then have that for any $z\in \mathcal{Z}$
\begin{align*}
   \mathbb{E}|f(x_T;z) - f(x'_T;z)| 
   &\le  2M\frac{t_0}{n} +\frac{\mu}{n} exp\left(\gamma \frac{c^2D_1}{2}+ (1-\beta) D_2 T\right)\times \\
   &\hspace{2cm}T^{c\sqrt{D_1}}
    \left(
        \frac{Y}{\sqrt{D_1}} +(1-\beta)Z\sum_{t = t_0+1}^T\frac{1}{t^{c\sqrt{D_1}}}
    \right).
\end{align*}

\end{corollary}
\begin{remark}
Although this bound is loose, it provides some insights about the stability of Adam with $\beta_1 = 0$. First, the stability bound depends on the parameter $\beta$. Specifically, the optimizer is more stable when $\beta$ is close to $1$ because it reduces the effect of the term $exp((1-\beta)D_2T)$ and the term $(1-\beta)Z\sum_{t = t_0+1}^T\frac{1}{t^{c\sqrt{D_1}}}$. Our experiments confirm this insight. Second, the stability bound also depends on the initial step size $c$. Small $c$ reduce the effect of the term $exp(\gamma\frac{c^2D_1}{2})$. 
\end{remark}
In order to overcome the exponential bound in the previous corollary, we can let $\{\beta_t\}$ gradually converge to $1$. The following corollary shows that this is indeed the case. 
\begin{corollary} \label{corollary1}
Let $\eta_t = \frac{c}{t}$ and $\beta_t = 1-\alpha_t , \alpha_t\in (0,1)$ such that $\alpha_t \le \frac{\alpha}{t}$ for all $t$. We then have that 
\begin{align*}
   \mathbb{E}|f(x_T;z)  - f(x'_T;z)|&\le2M\frac{t_0}{n} 
   + \frac{\mu}{n}(cY +  \alpha Z) \exp\left(\gamma\frac{c^2D_1}{2}\right) \times \frac{T^{c\sqrt{D_1}+\alpha D_2} }{c\sqrt{D_1}+\alpha D_2} \times \frac{1}{t_0^{c\sqrt{D_1}+\alpha D_2}}.  
\end{align*}
\end{corollary}
\vspace{-2mm}                                            
We could actually choose $t_0$ to minimize the bound in the above corollary as in  \cite{hardt2016train}. We have the following corollary. 
\begin{corollary} \label{corollary2}
Let $\eta_t = \frac{c}{t}$ and $\beta_t = 1-\alpha_t , \alpha_t\in (0,1)$ such that $\alpha_t \le \frac{\alpha}{t}$ for all $t$. Then the uniform stability error of AOM is given by
\begin{align*}
   \mathbb{E}|& f(x_T;z)  - f(x'_T;z)|\le \frac{1}{n}\left[2M\left(\frac{AB}{2M}\right)^{\frac{1}{A+1}} + \frac{B}{\left(\frac{AB}{2M}\right)^{\frac{A^2}{A+1}}}\right]T^{\frac{A}{A+1}} 
\end{align*}
where we define 
\begin{align*}
A:= c\sqrt{D_1}+\alpha D_2,\quad\quad\quad\quad\quad\quad\quad\quad
B:= (cY + \alpha Z) exp\left(\gamma\frac{c^2D_1}{2}\right)\frac{1}{c\sqrt{D_1}+\alpha D_2}.
\end{align*}
\end{corollary}

\begin{remark}
When $\beta_t = 1-\frac{\alpha}{t}$, we obtain Adagrad algorithm. Thus, this corollary implies that Adagrad algorithm is uniformly stable at rate $O(T^{r}),0< r<1$. To the best of our knowledge, this is the first result which shows that Adagrad algorithm is a uniformly stable algorithm. Our experiments confirm the rate in the above corollary.
\end{remark}

\vspace*{-\baselineskip}
\section{Stability of adaptive optimization methods  with weight decay} \label{sec:stability}
\vspace{-2mm}
Weight decay is one of common techniques when training neural network~\citep{loshchilov2019decoupled}. In this section, we prove that weight decay can actually help improve the stability of adaptive optimization methods. Our analysis focus on AdamW algorithm. Remind that the update of AdamW on two data set $S$ and $S'$ has the form

\begin{minipage}{0.48\textwidth}
\begin{align}
    x_{t+1} &= (1-\eta_t\lambda)x_t - \eta_t \frac{\nabla f(x_t,z_{i_t})}{\sqrt{v_{t+1}+\epsilon}} 
    \\
    v_{t+1} &= \beta_t v_t + (1-\beta_t)\nabla f(x_t,z_{i_t})^2
\end{align}
\end{minipage}
\hfill
\begin{minipage}{0.49\textwidth}
\begin{align}
    x'_{t+1} &=(1- \eta_t\lambda)x'_t - \eta_t \frac{\nabla f(x'_t,z'_{i_t})}{\sqrt{v'_{t+1}+\epsilon}}
    \\
    v'_{t+1} &= \beta_t v'_t + (1-\beta_t)\nabla f(x'_t,z'_{i_t})^2
\end{align}
\end{minipage}

where $\lambda$ is the weight decay parameter. We have the following theorem. 
\begin{theorem} \label{theorem5}
Assume that assumptions (1), (2), (3) are satisfied. Suppose we run AdamW with step size $\eta_t$ and weight decay $\lambda$ on $S$ and $S'$ respectively. We then have that 
\begin{align*}
    \Delta_{t+1} \le (1-\eta_t\lambda)\Delta_t  &+ \eta_t \left(1-\frac{1}{n}\right)\frac{L}{\sqrt{\lambda_1+\epsilon}}\left[1 + \frac{\mu L(1-\beta_t)}{\lambda_1+\epsilon}\right]\Delta_t\\
    & + \eta_t \left(1-\frac{1}{n}\right)\frac{\mu\beta_t}{2\sqrt{\lambda_1+\epsilon}(\lambda_1+\epsilon)}\Sigma_t + \eta_t\frac{1}{n}\frac{2}{\sqrt{\lambda_1+\epsilon}}\mu.
\end{align*}
\end{theorem}
The proof of this theorem is similar to the proof of Theorem \ref{theorem2}. 
From now on, we fix $\beta_t = \beta, \eta_t = \eta$, the this inequality then becomes
\begin{align}
\Delta_{t+1} \le 
    \biggl\{1 - \eta \biggl(\lambda -& \left(1-\frac{1}{n}\right)\frac{L}{\sqrt{\lambda_1+\epsilon}}  
        \left[1 + \frac{\mu L(1-\beta)}{\lambda_1+\epsilon}\right]\biggl)
    \biggl\}
    \Delta_t
    \nonumber \\
& + \eta \left(1-\frac{1}{n}\right)\frac{\mu\beta}{2\sqrt{\lambda_1+\epsilon}(\lambda_1+\epsilon)}\Sigma_t + \eta\frac{1}{n}\frac{2}{\sqrt{\lambda_1+\epsilon}}\mu.
\end{align}
In addition, we also have that 
$
\Sigma_{t+1} \le \beta \Sigma_t +  (1- \beta)\left(1-\frac{1}{n}\right)2L\mu\Delta_t + \frac{1}{n} (1-\beta)2\mu^2.  
$
Define 
\begin{align*}
U := \left(1-\frac{1}{n}\right)\frac{L}{\sqrt{\lambda_1+\epsilon}}\left[1 + \frac{\mu L(1-\beta)}{\lambda_1+\epsilon}\right],  \quad
&V := \left(1-\frac{1}{n}\right)\frac{\mu\beta}{2\sqrt{\lambda_1+\epsilon}(\lambda_1+\epsilon)},\quad\\
W := \left(1-\frac{1}{n}\right)2L\mu.
\end{align*}
We then can rewrite our system of inequalities as 
\begin{align}
\begin{split}
\begin{bmatrix}
\Delta_{t+1}\\
\Sigma_{t+1}
\end{bmatrix}  & \le  
\begin{bmatrix}
1 - \eta(\lambda - U) & \eta V \\ 
(1-\beta)W& \beta 
\end{bmatrix} \begin{bmatrix}
\Delta_t\\
\Sigma_t
\end{bmatrix} + \frac{1}{n}\begin{bmatrix}
Y\\
Z\\
\end{bmatrix},  
\end{split}
\end{align}
Denote 
$$ A_t :=   \begin{bmatrix}
\Delta_t\\
\Sigma_t
\end{bmatrix}\quad B := \begin{bmatrix}
Y\\
Z\\
\end{bmatrix} \quad R := \begin{bmatrix}
1 - \eta(\lambda - U) & \eta V \\ 
(1-\beta)W& \beta 
\end{bmatrix},$$
we then have that 
\begin{align}
\|R\|^2_F & \le (1 - \eta(\lambda - U))^2 +  \eta^2 V^2 + (1-\beta)^2W^2 + \beta^2. 
\end{align}
First, we show that  the inequality 
$(1-\beta)^2W^2 + \beta^2  < 1$
has a solution $\beta\in (0,1)$. By direct computation, we can show that the equation 
$(1-\beta)^2W^2 + \beta^2  = 1$
has two distinct solutions $\beta_1 = 1$ and $\beta_2 = \frac{W^2 -1}{W^2 +1}$. If $W^2>1$, then we have that $0 < \beta_2 < \beta_1 = 1$. On the other hand, if $W^2 < 1$, then we have that $\beta_2 < 0 < \beta_1 = 1$. 
Thus, we conclude that for any $\beta \in (\max\{0, \beta_2\}, 1)$, we always have that 
$(1-\beta)^2W^2 + \beta^2  < 1.$

Now, let's assume that we choose $\beta$ such that the above inequality holds. We can then choose $\eta$ small enough such that $\eta^2 V^2 + (1-\beta)^2W^2 + \beta^2 < 1.$
Given $\beta, \eta$, we want to choose $\lambda$ such that
$(1 - \eta(\lambda - U))^2 +  \eta^2 V^2 + (1-\beta)^2W^2 + \beta^2 < 1.$
First, we need $1 - \eta(\lambda- U) > 0$, which is equivalent to $\lambda < U + \frac{1}{\eta}.$
In addition, we also want
$(1 - \eta(\lambda - U))^2 +  \eta^2 V^2 + (1-\beta)^2W^2 + \beta^2 < 1,$
or equivalently, 
$
\lambda > U + \frac{1- \sqrt{1-\eta^2 V^2 - (1-\beta)^2W^2 - \beta^2}}{\eta}
$.
Thus we can find $\lambda$ such that $\|R\|_F < 1$. With this choice of $\lambda$, we have that 
\begin{align}
    \|A_{t+1}\| \le \alpha \|A_t\| + \frac{1}{n}B\le \frac{1}{n}\frac{B}{1-\alpha},
\end{align}
where $\alpha = \|R\|_F$. 

The above derivation shows that with an appropriate choice of $\beta, \eta, \lambda$,  AdamW is a uniformly stable algorithm. In addition, the stability bound is independent of the number of iterations.

\section{Experiments} \label{sec:exp}
\vspace*{-2mm}
We study the generalization and stability of AOMs with different angles using synthetic to real-world datasets. Following \cite{hardt2016train}, for each training dataset, we first remove a random sample $x$ and make two copies of the training set. Then for the first training set, we randomly replace a data point with $x$. We then train two models on these two train sets staring from the same initialization. At each iteration, we record the Euclidean distance between the parameters of the two output models. We also record the training loss and test loss of the first model, from which we calculate the generalization error which is the absolute difference between train and test losses. We note that for each metric, we run 20 trials and plot the corresponding mean and variation. All the codes are included in the supplementary material and will be publicly released.

\vspace*{-2mm}
\subsection{Synthetic data} 
\vspace*{-2mm}

\begin{figure*}[ht!]
\begin{subfigure}{.24\textwidth}
  \centering
  \includegraphics[width=0.89\linewidth]{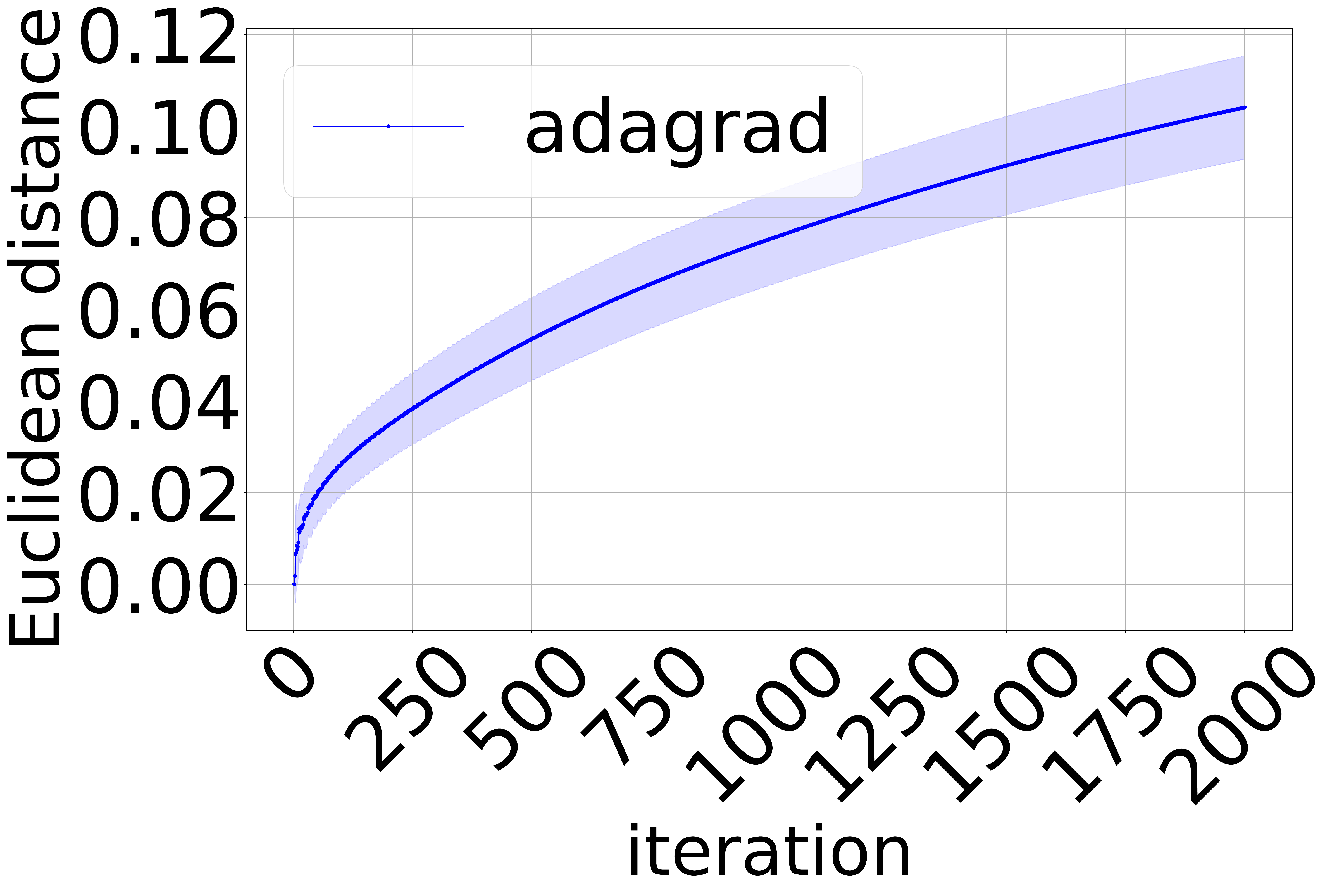}  
  \label{fig:sub-first}
\end{subfigure}
\begin{subfigure}{.24\textwidth}
  \centering
  \includegraphics[width=0.89\linewidth]{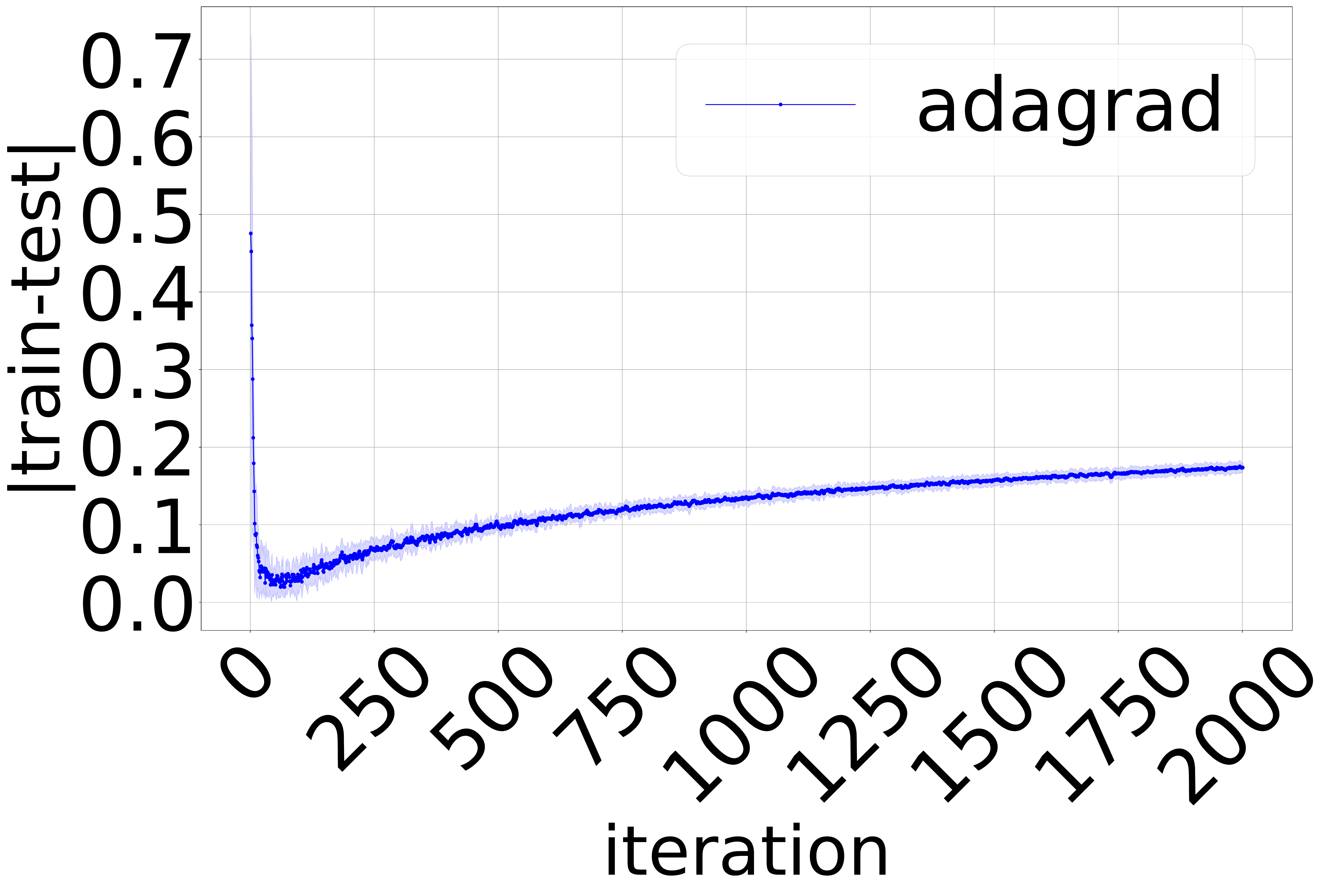}  
  \label{fig:sub-second}
\end{subfigure}
\begin{subfigure}{.24\textwidth}
  \centering
  \includegraphics[width=0.89\linewidth]{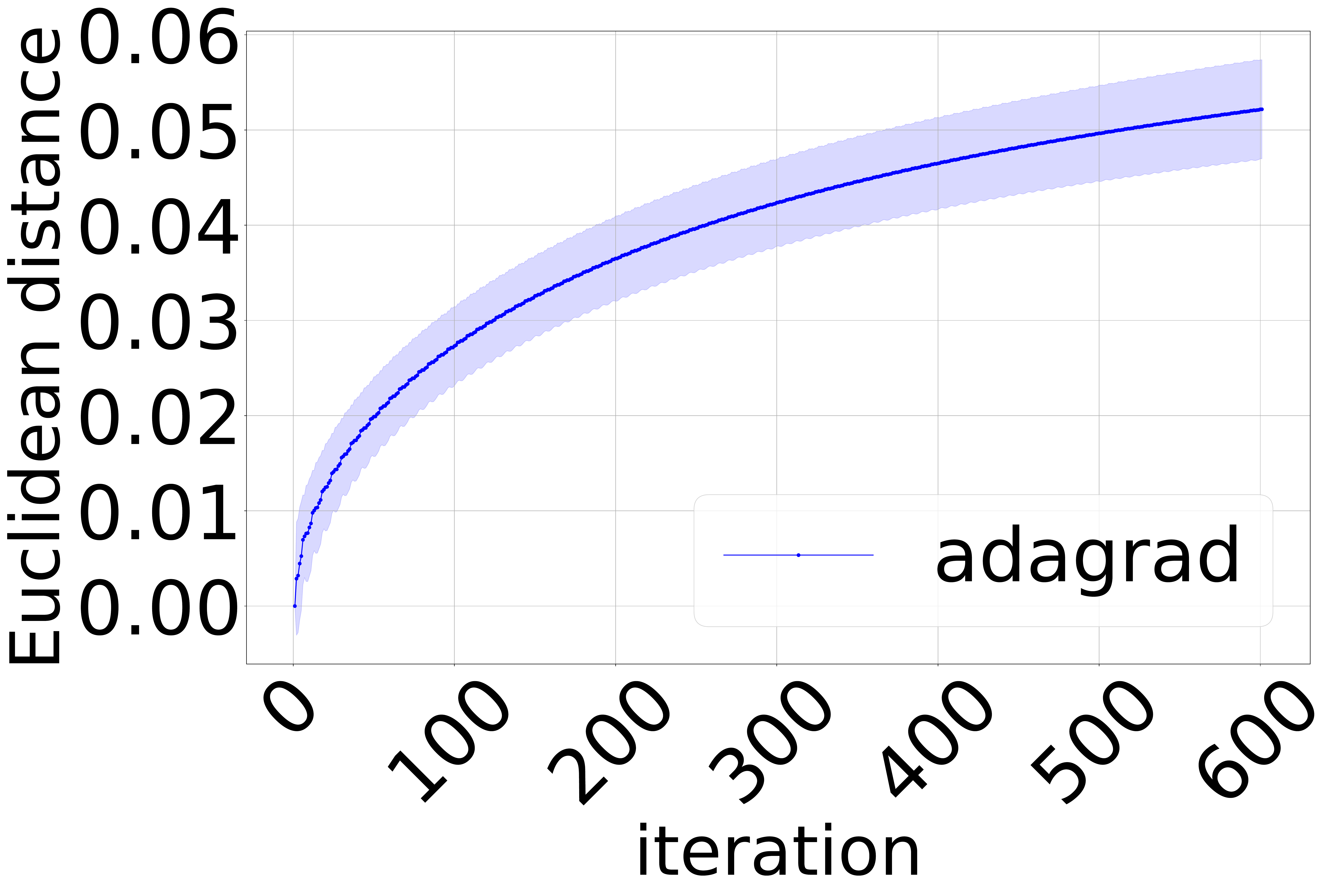}  
  \label{fig:sub-first}
\end{subfigure}
\begin{subfigure}{.24\textwidth}
  \centering
  \includegraphics[width=0.89\linewidth]{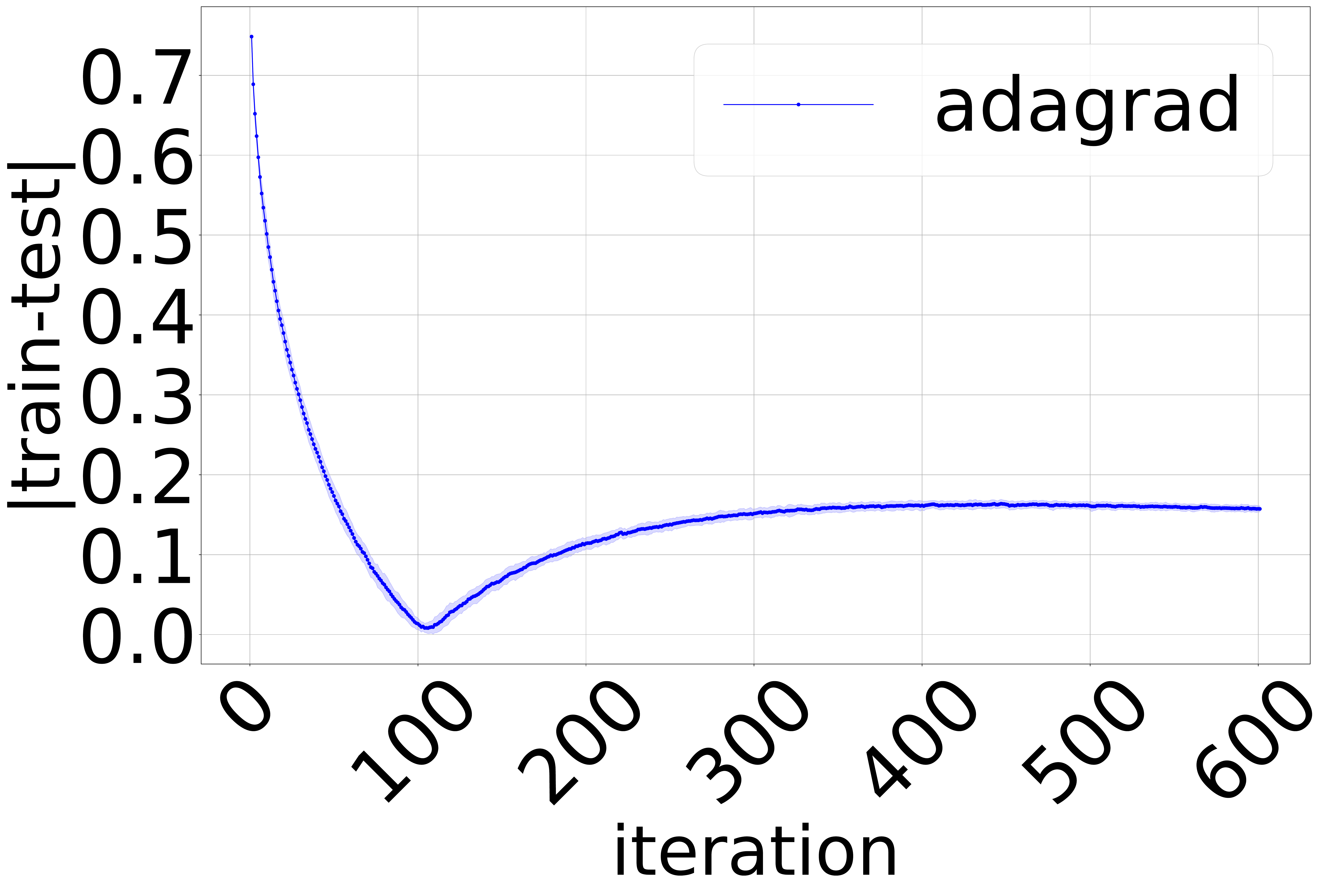}  
  \label{fig:sub-second}
\end{subfigure}
\caption{
Parameter distance and generalization error when using Adagrad to train neural networks to solve CLS (left) and REG (right) tasks. Both metrics grow in similar fashion, which agree with Corollary~\ref{corollary2}. 
}
\label{fig:adagrad}
\vspace*{-2mm}
\end{figure*}

\begin{figure*}[ht!]
\vspace*{-2mm}
\begin{subfigure}{.24\textwidth}
  \centering
  \includegraphics[width=0.89\linewidth]{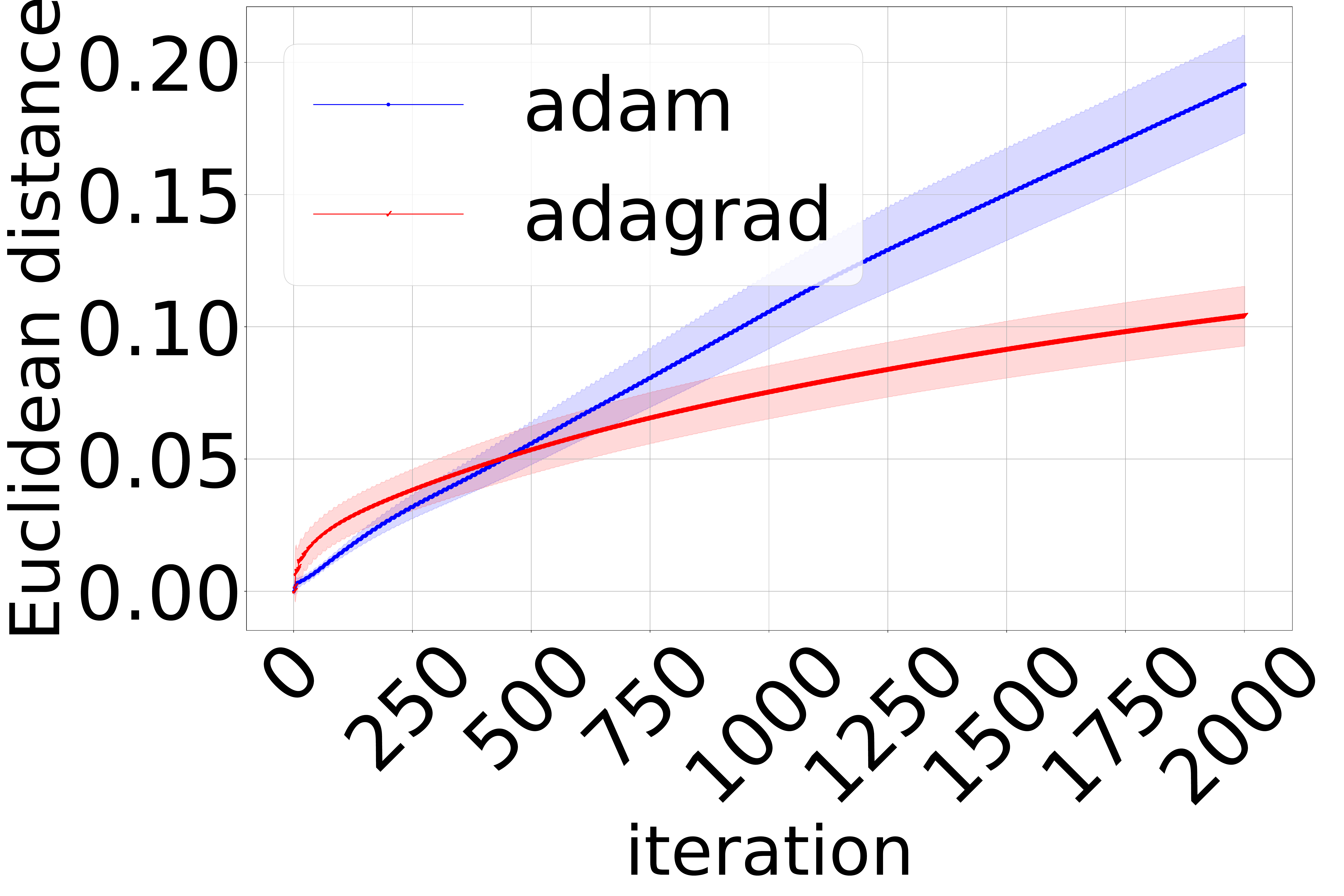}  
  \label{fig:sub-first}
\end{subfigure}
\begin{subfigure}{.24\textwidth}
  \centering
  \includegraphics[width=0.89\linewidth]{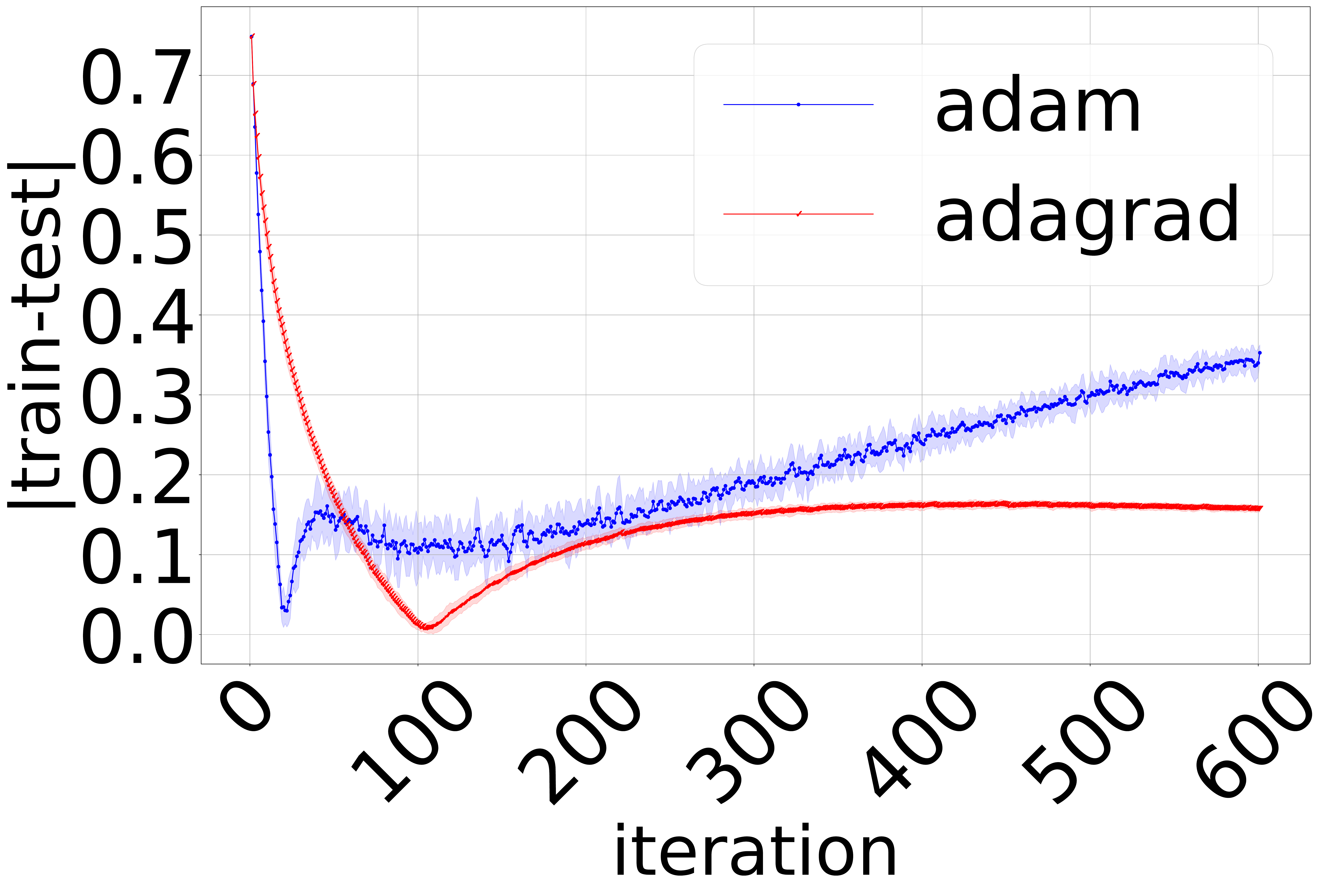}  
  \label{fig:sub-second}
\end{subfigure}
\begin{subfigure}{.24\textwidth}
  \centering
  \includegraphics[width=0.89\linewidth]{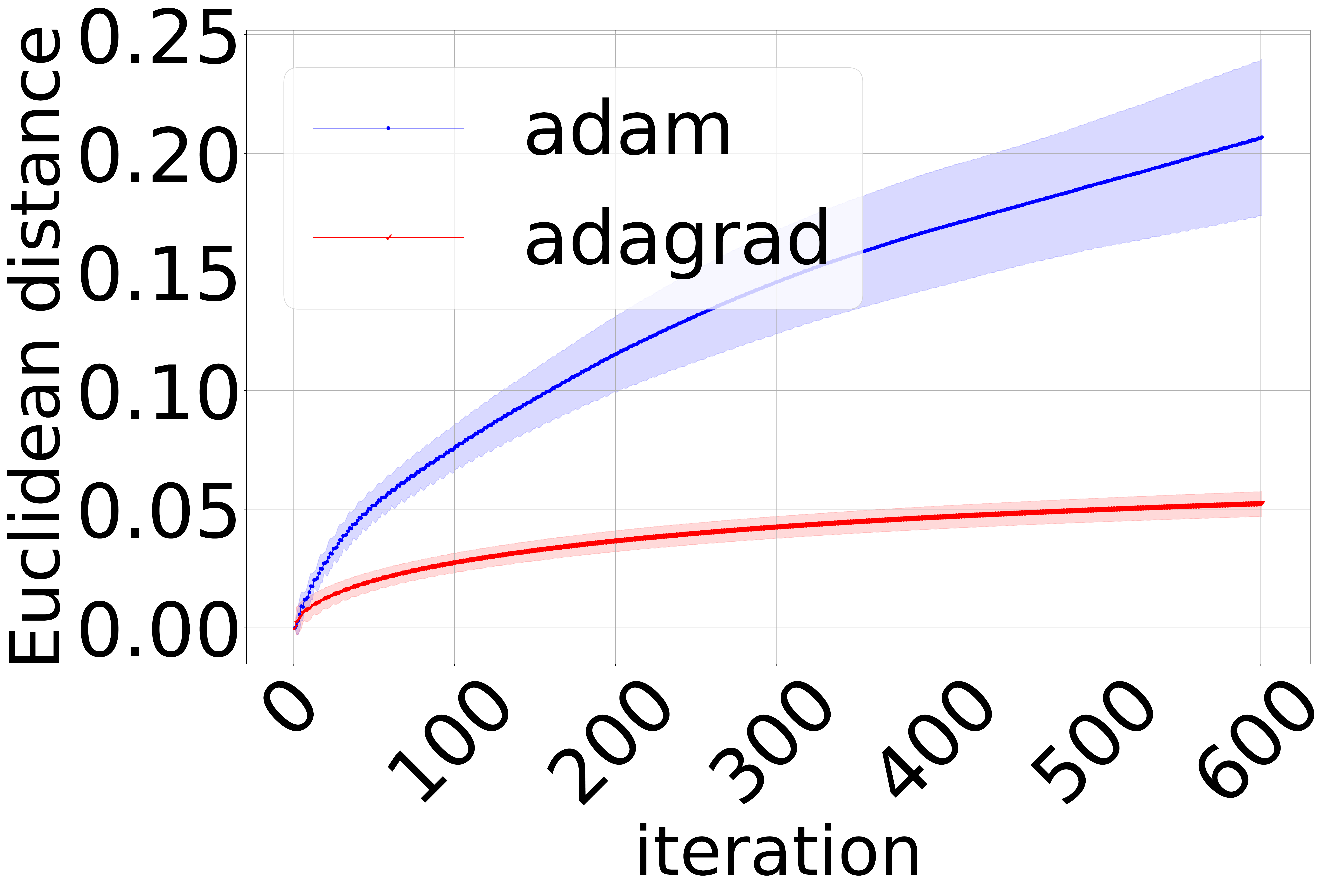}  
  \label{fig:sub-first}
\end{subfigure}
\begin{subfigure}{.24\textwidth}
  \centering
  \includegraphics[width=0.89\linewidth]{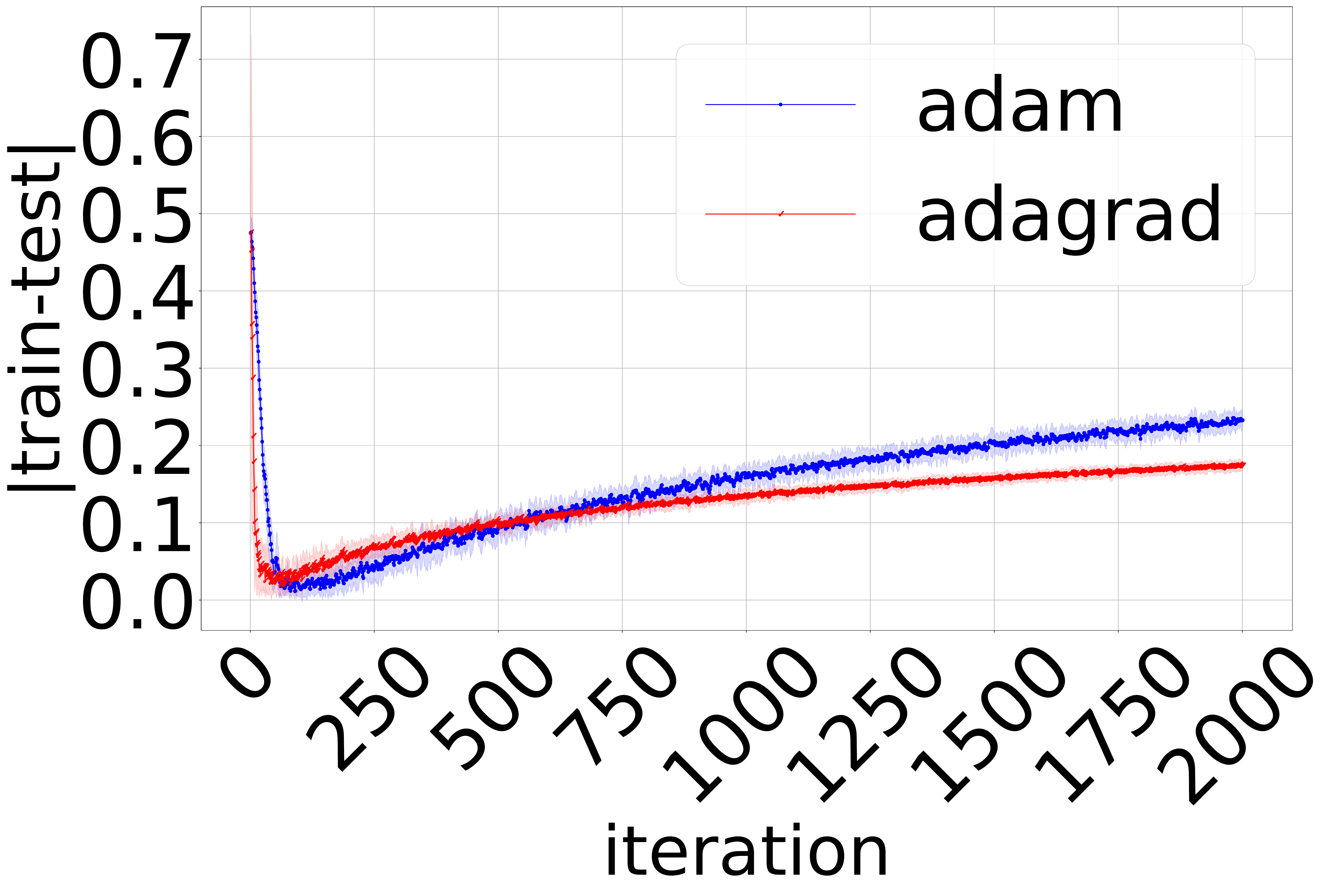}  
  \label{fig:sub-second}
\end{subfigure}
\caption{Comparing parameter distance and generalization error between Adam and Adagrad in the classification (left) and regression (right). Results show that Adagrad is more stable than Adam.} 
\label{fig:adagrad_adam}
\vspace*{-4mm}
\end{figure*}

\textbf{Datasets and Tasks.}  We consider both classification (CLS) and regression (REG) tasks. The details for network architecture, training parameters and additional results can be found in Appendix~\ref{appx:synthetic}. Data generation process is as follows.  
\begin{enumerate}[nosep,noitemsep,nolistsep]
    \item For CLS, we generate 151 data points in 2D from three isotropic Gaussian blobs (3 classes) where each blob has a standard deviation $1$ and the same number of data points. We use the package  \emph{sklearn.datasets.make\_blobs} for this data generation procedure. We then randomly split the data set into an original train set (with 31 data points) and a test set (with 120 data points). 
     \item For REG, we generate 151 data points from the uniform distribution on $[-1,1]\times[-1,1]$. The targets are then generated by the following formula: $y = x_1^2+x_2^2 + \mathbf{\epsilon}
     $ where the noise $\epsilon \sim \mathcal{N}(0,0.5^2)$. We then also randomly split the data set into an original train set (with 31 data points) and a test set (with 120 data points).
\end{enumerate}
We choose a small training set in both tasks because we want our models to overfit the training data easily. Thus, we can compare the generalization errors between different algorithms.

\textbf{Stability and generalization of Adagrad.} 
In this experiment, we consider AOMs with $\alpha_t = 0, \beta_t = 1-\alpha/t$. This corresponds to Adagrad. For comparison with Adagrad, we consider Adam with $\beta_1=0, \beta_2 = 0.999$.
As you can see from Figure~\ref{fig:adagrad}, the parameter distance in both tasks grow as $O(T^{r})$ where $r < 1$. This aligns with the result in Corollary~\ref{corollary1}. The generalization error in both tasks also grows at a similar rate. 
Additionally, in Figure~\ref{fig:adagrad_adam}, we can also see that the parameter distance and the generalization error of Adagrad  always grow slower than the ones of Adam although we use the same or bigger initial learning rate for Adagrad. This shows that letting $\beta_t$ goes to 1 at rate $1-\alpha/t$ make the AOM more stable than using fixed $\beta_t= 0.999$.
\textbf{Dependence of stability and generalization on fixed $\beta_t$.} 
In this experiment, we study the stability and generalization of AOMs when $\alpha_t = 0$ and $\beta_t = \beta$. This corresponds to  Adam with $\beta_1 = 0$ and $\beta_2 = \beta$. 

    
\begin{figure*}[ht]
\vspace*{-2mm}
\begin{subfigure}{.24\textwidth}
  \centering
  \includegraphics[width=0.89\linewidth]{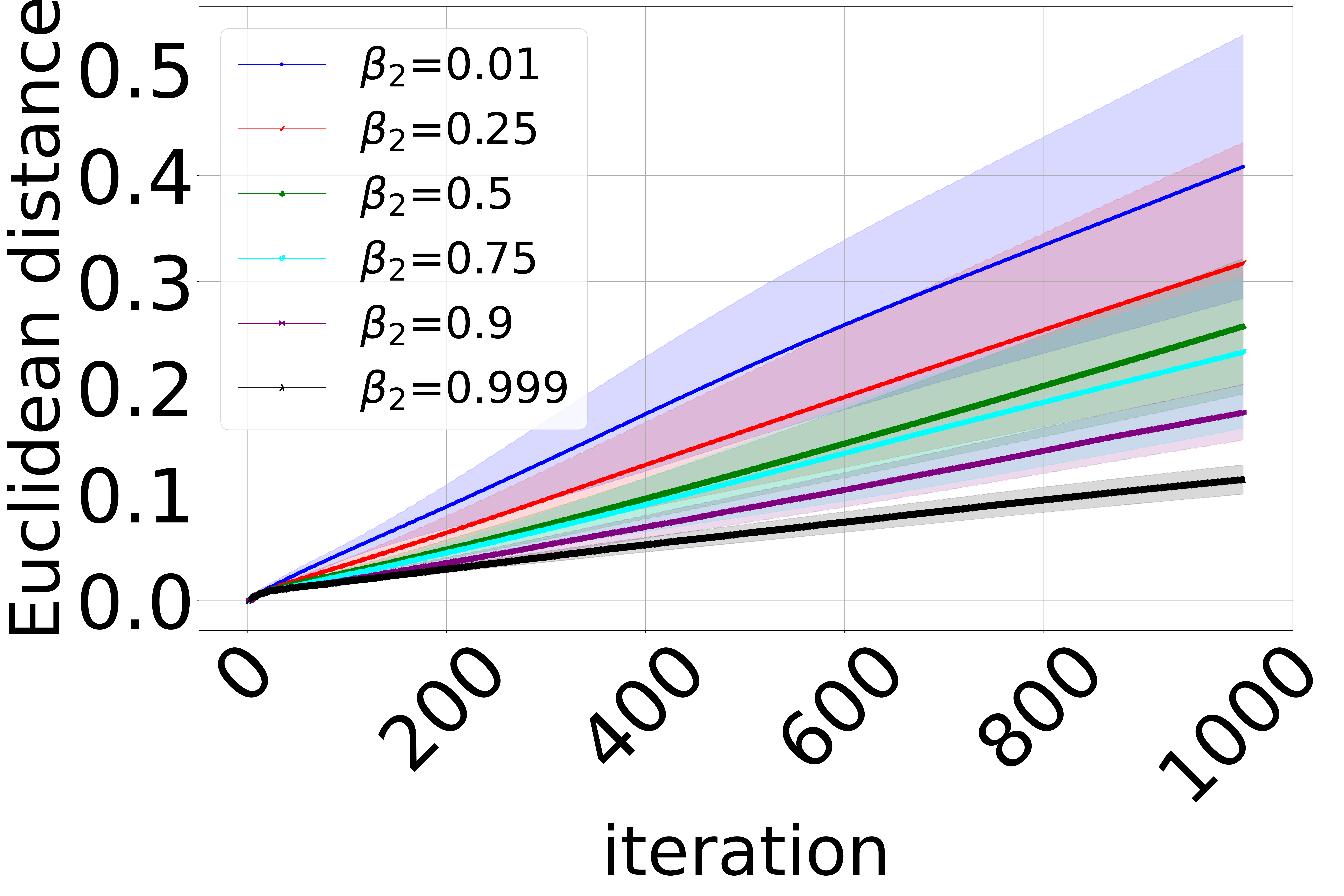}  
  \label{fig:sub-first}
\end{subfigure}
\begin{subfigure}{.24\textwidth}
  \centering
  \includegraphics[width=0.89\linewidth]{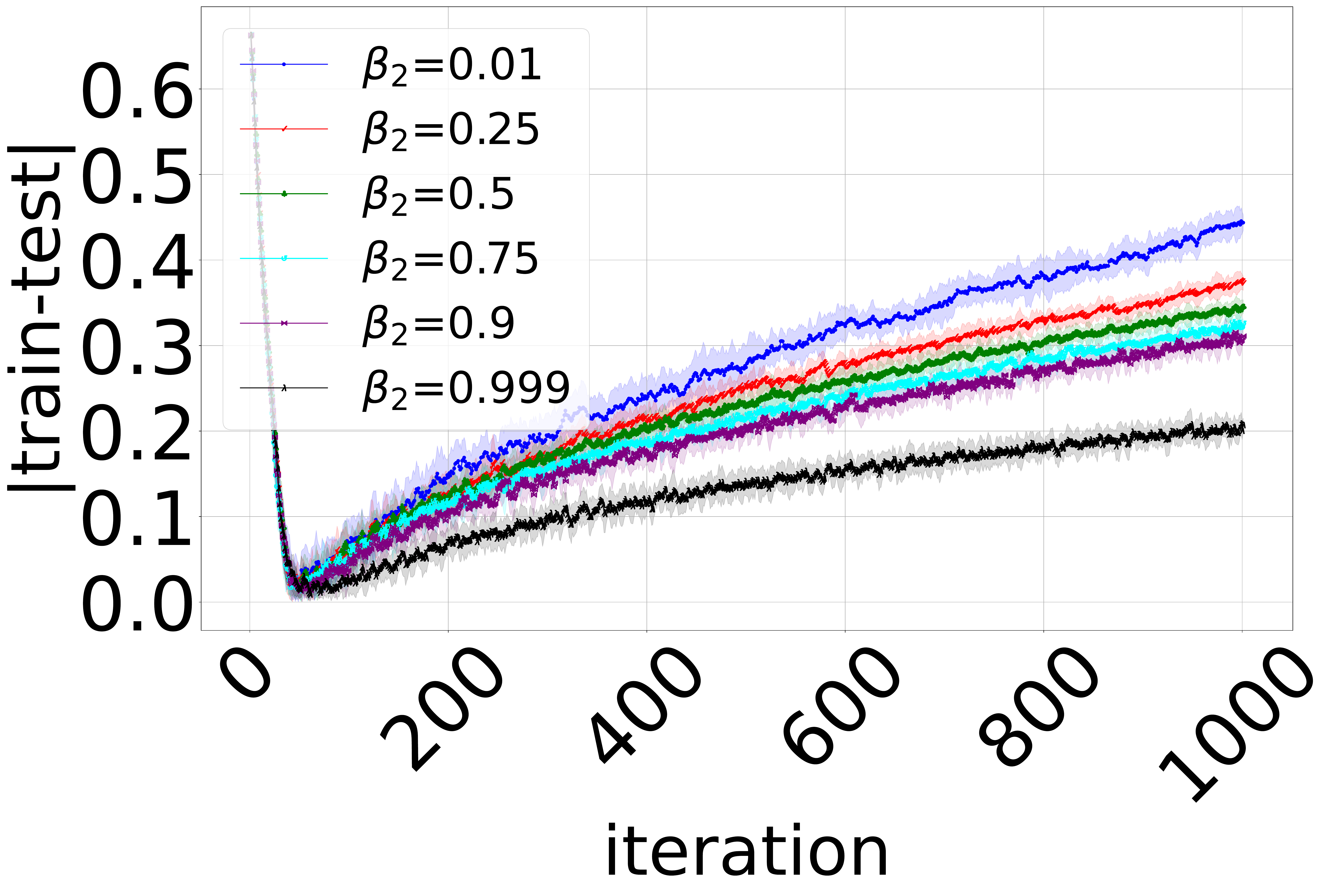}  
  \label{fig:sub-second}
\end{subfigure}
\begin{subfigure}{.24\textwidth}
  \centering
  \includegraphics[width=0.89\linewidth]{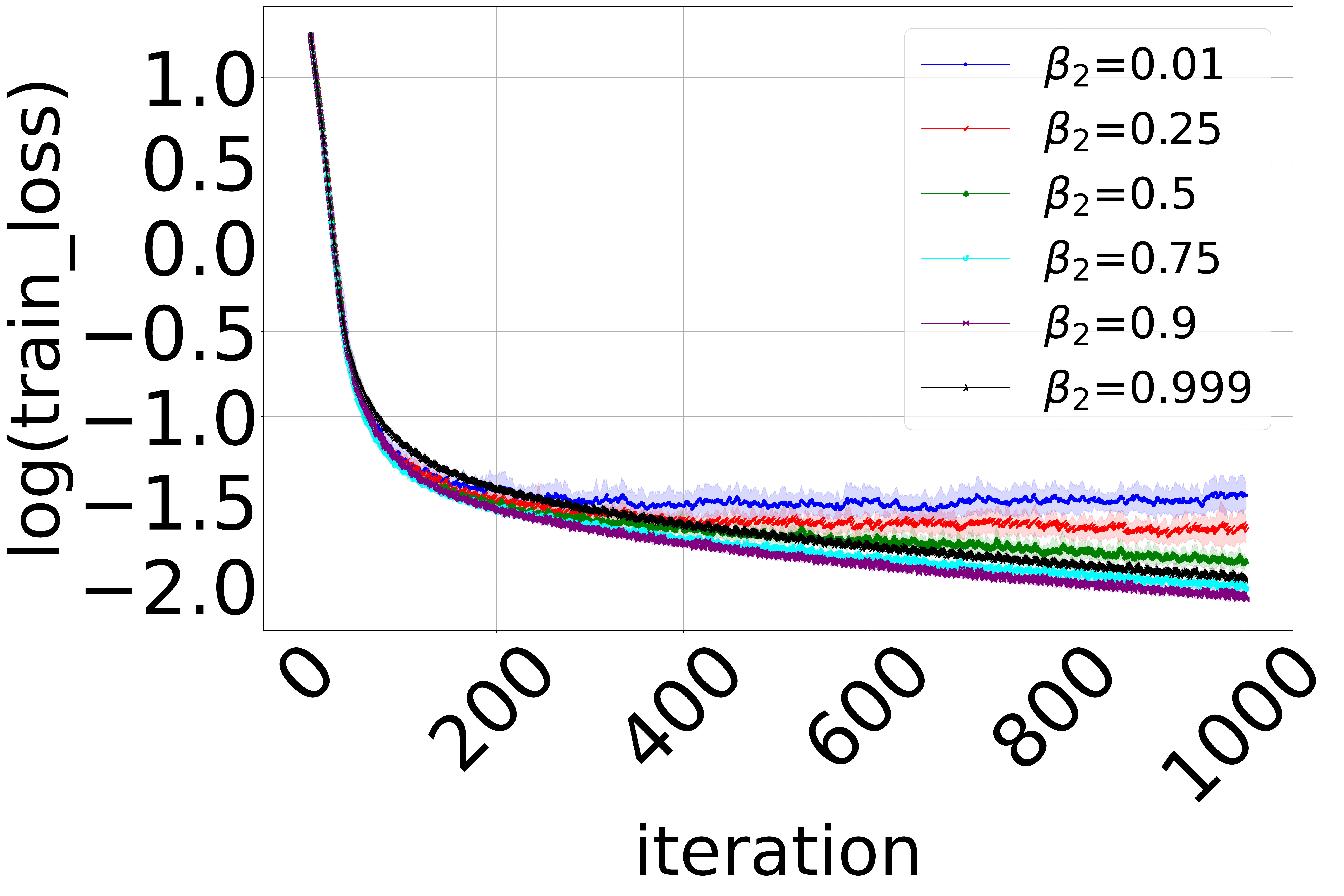}  
  \label{fig:sub-first}
\end{subfigure}
\begin{subfigure}{.24\textwidth}
  \centering
  \includegraphics[width=0.89\linewidth]{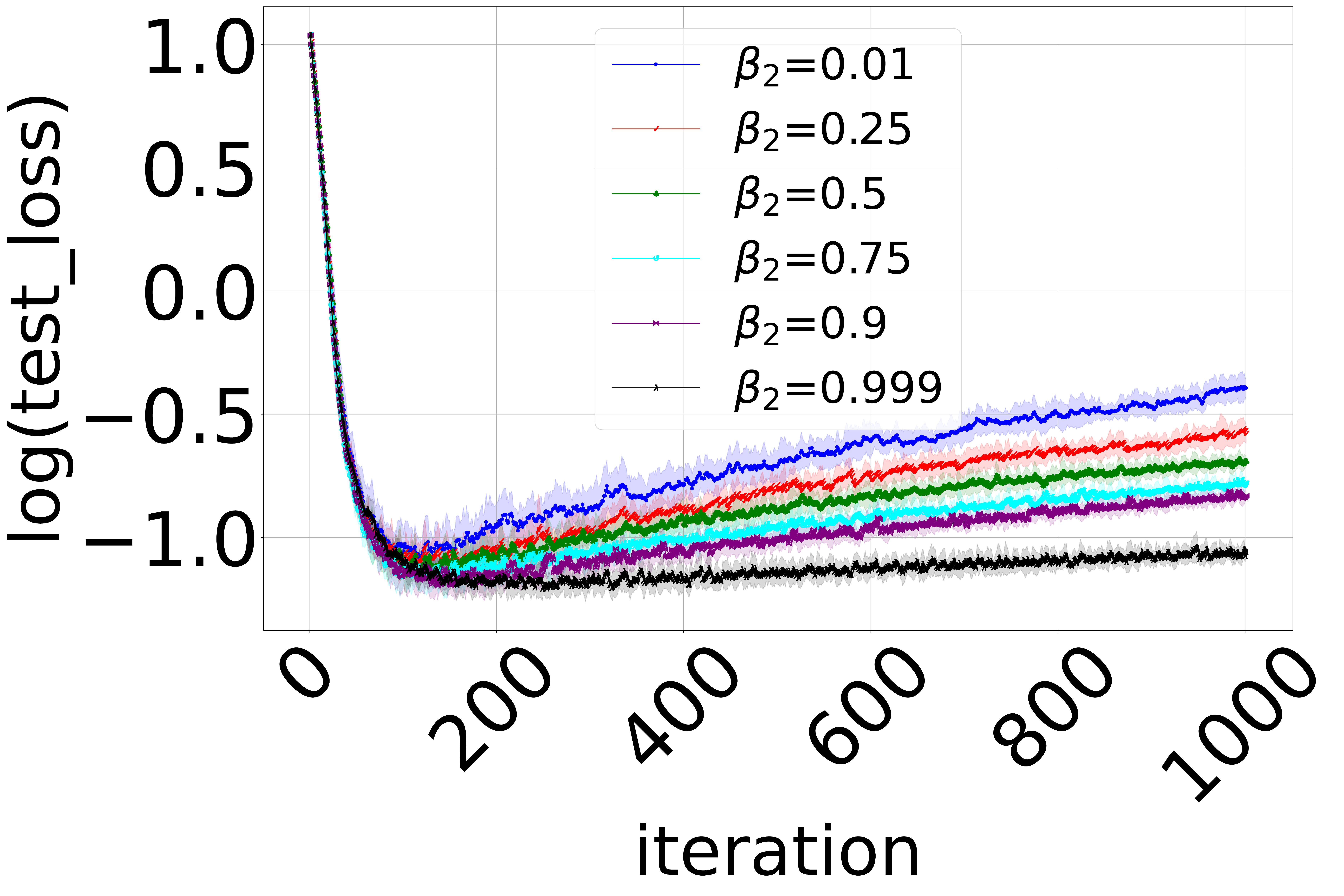}  
  \label{fig:sub-second}
\end{subfigure}
\\
\begin{subfigure}{.24\textwidth}
  \centering
  \includegraphics[width=0.89\linewidth]{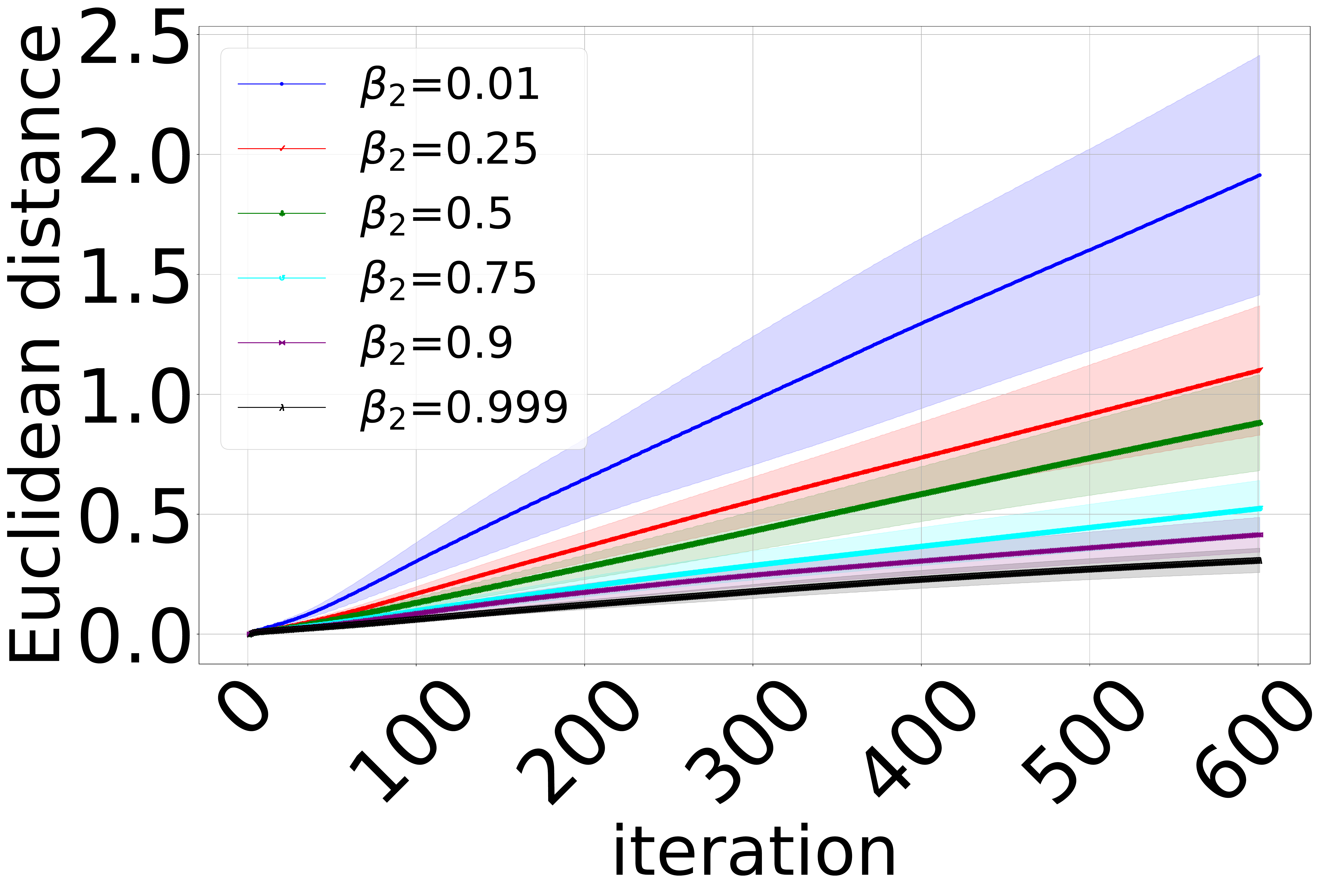}  
  \label{fig:sub-first}
\end{subfigure}
\begin{subfigure}{.24\textwidth}
  \centering
  \includegraphics[width=0.89\linewidth]{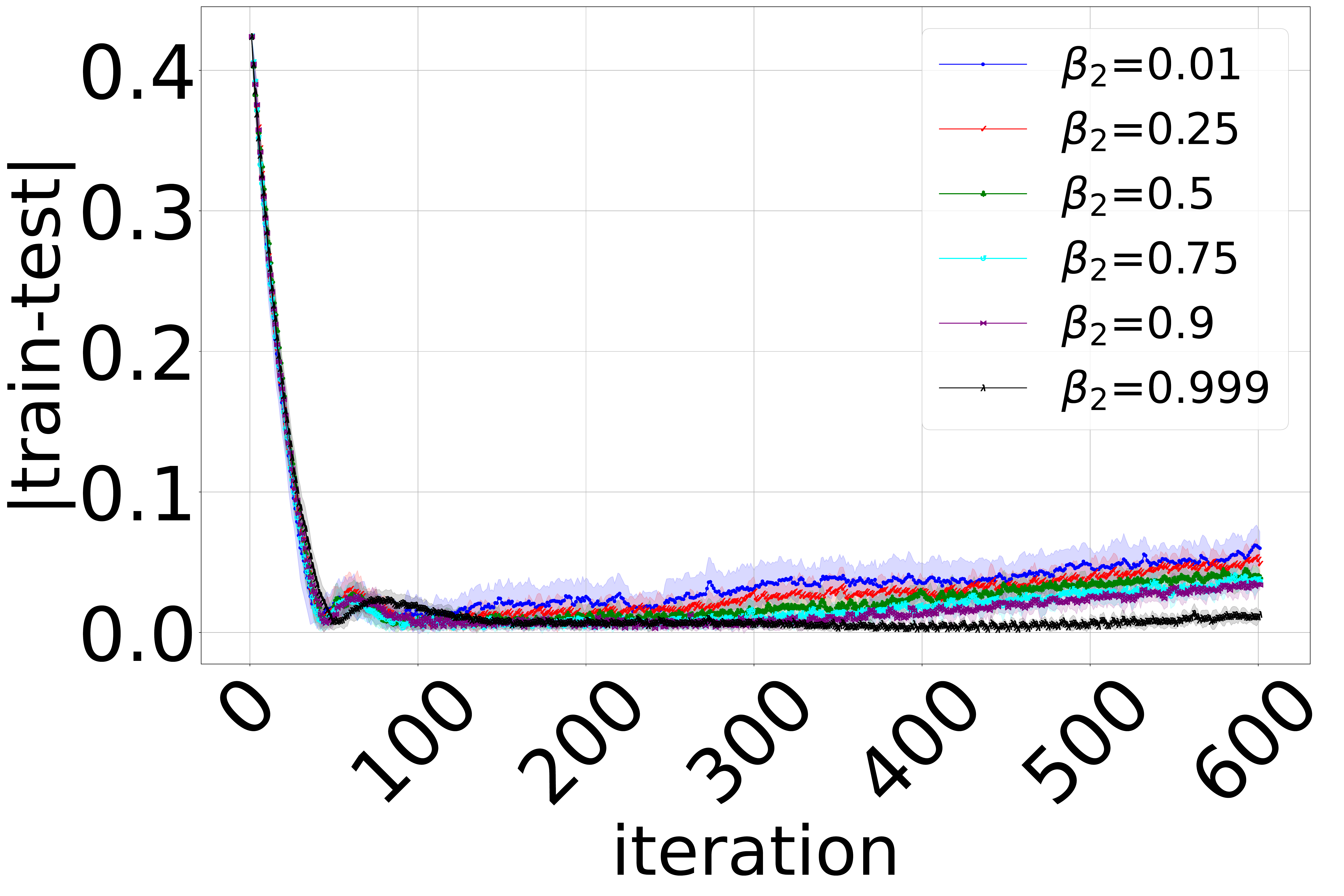}  
  \label{fig:sub-second}
\end{subfigure}
\begin{subfigure}{.24\textwidth}
  \centering
  \includegraphics[width=0.89\linewidth]{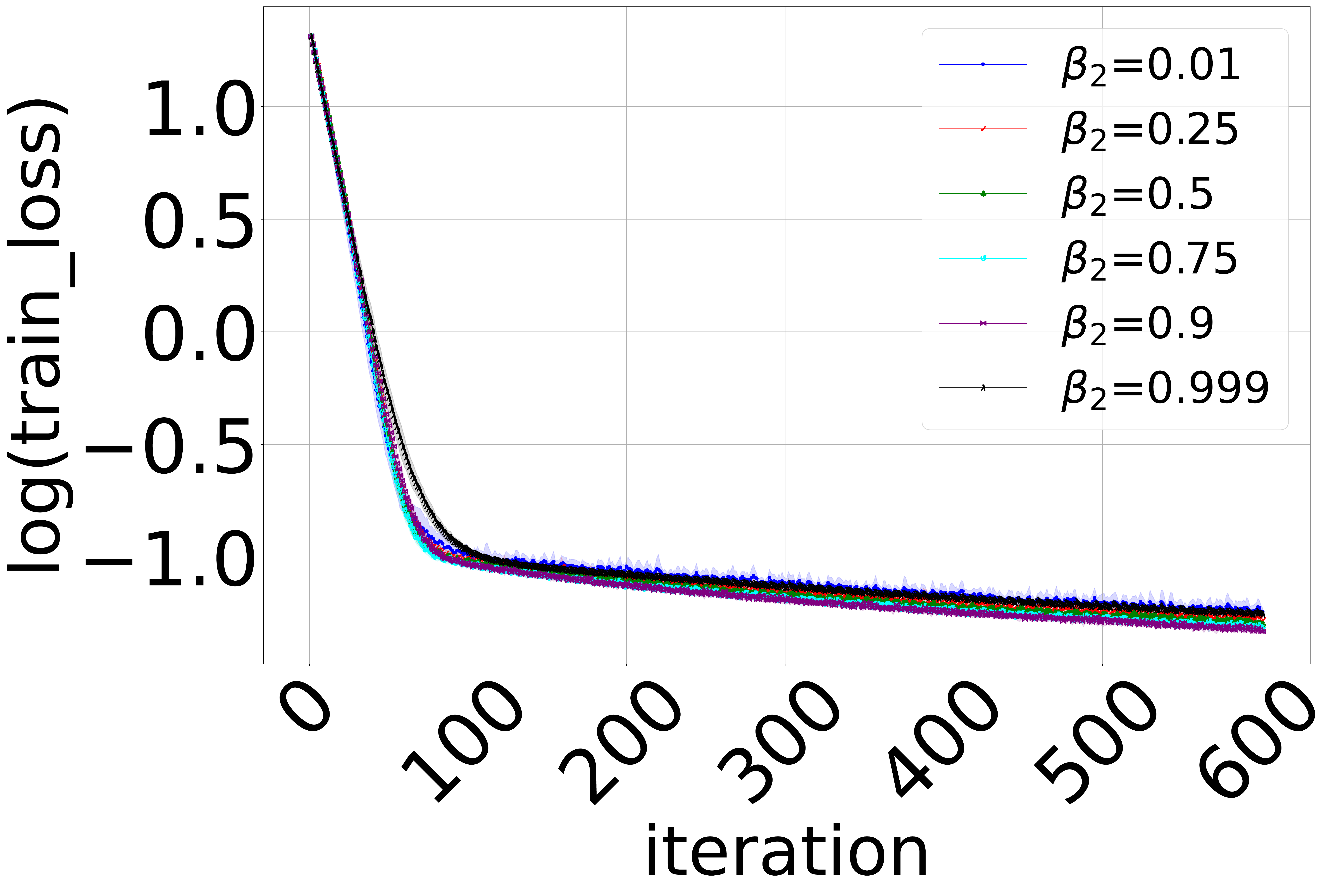}  
  \label{fig:sub-first}
\end{subfigure}
\begin{subfigure}{.24\textwidth}
  \centering
  \includegraphics[width=0.89\linewidth]{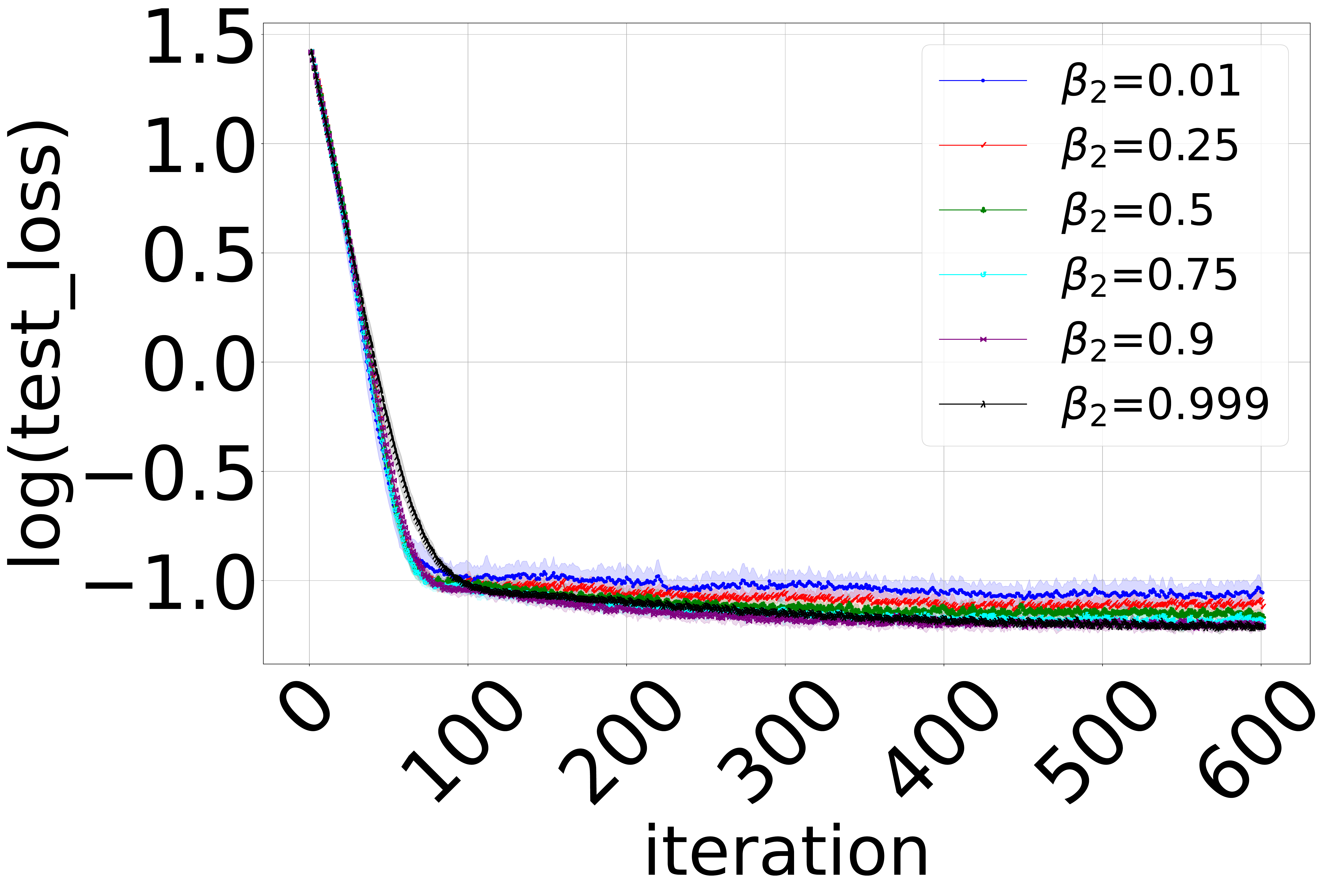}  
  \label{fig:sub-second}
\end{subfigure}
\caption{Performances when solving the CLS (top) and REG (bottom) using Adam with different $\beta_2$. The bigger $\beta_2$ is, the more stable the model gets. \textit{Left to right}: parameter distance, generalization error, train loss and test loss. The losses are plotted in log space. } 
\label{fig:adam_stability}
\vspace*{-2mm}
\end{figure*}


From Figure \ref{fig:adam_stability}, 
we can see that the stability of Adam depends on the parameter $\beta_2$, in which the bigger the $\beta_2$ is, the more stable the algorithm is. The generalization error also exhibits the same pattern: bigger $\beta_2$ gives smaller generalization error. This observation aligns with Corollary~\ref{coro} although the bound there seems to be loose.


\textbf{Weight decay helps improve the stability and generalization of Adam.} 
In this experiment, we study how weight decay affects the stability and generalization of Adam in each of CLS and REG tasks, in which the weight decay is set at $5.0$ and $1.0$, respectively.
For Adam with weight decay, we use AdamW~\cite{loshchilov2017decoupled}, and both optimizers use $\beta_1= 0, \beta_2 = 0.999$.
    
\begin{figure*}[ht]
\vspace*{-2mm}
\begin{subfigure}{.24\textwidth}
  \centering
  \includegraphics[width=0.89\linewidth]{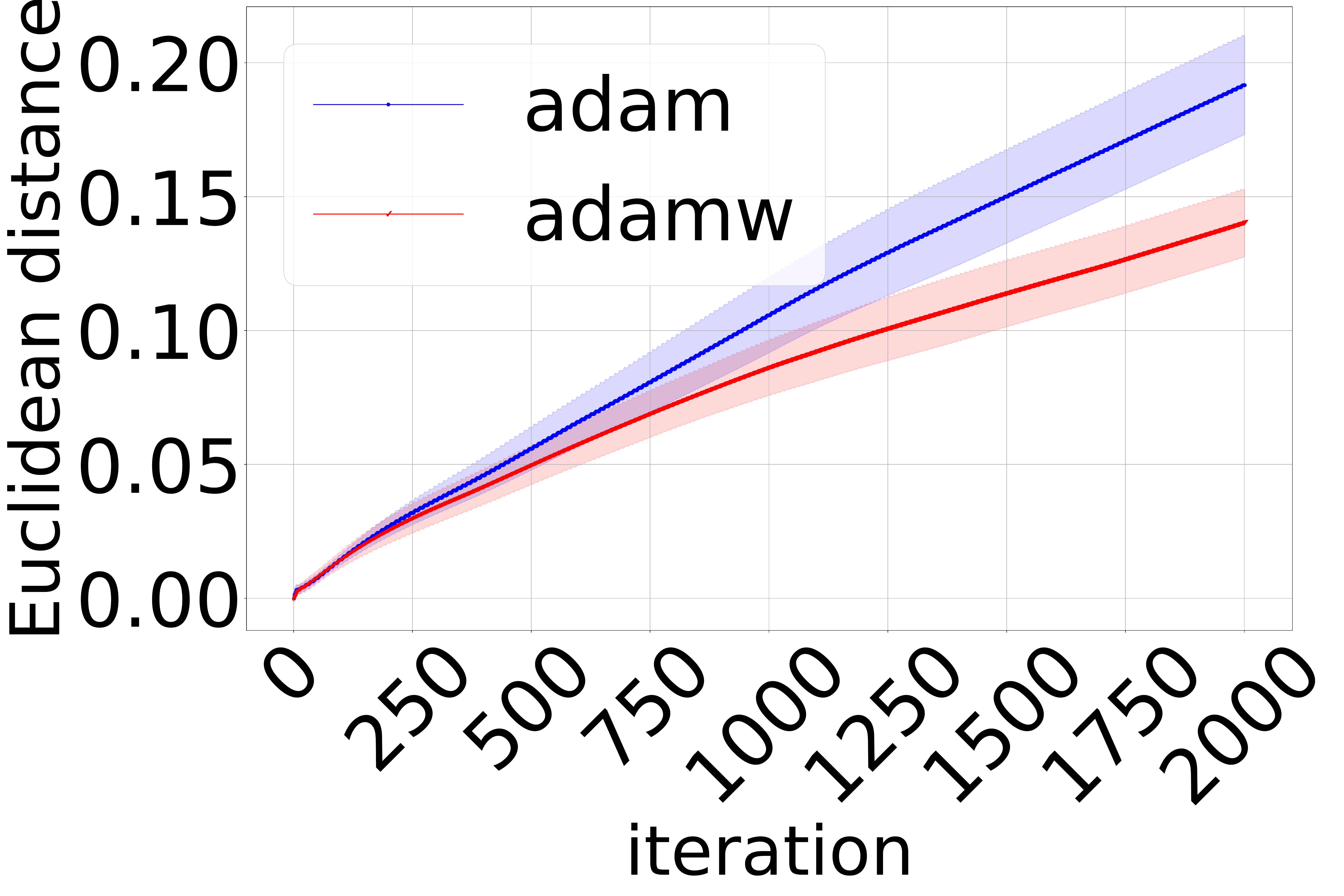}  
  \label{fig:sub-first}
\end{subfigure}
\begin{subfigure}{.24\textwidth}
  \centering
  \includegraphics[width=0.89\linewidth]{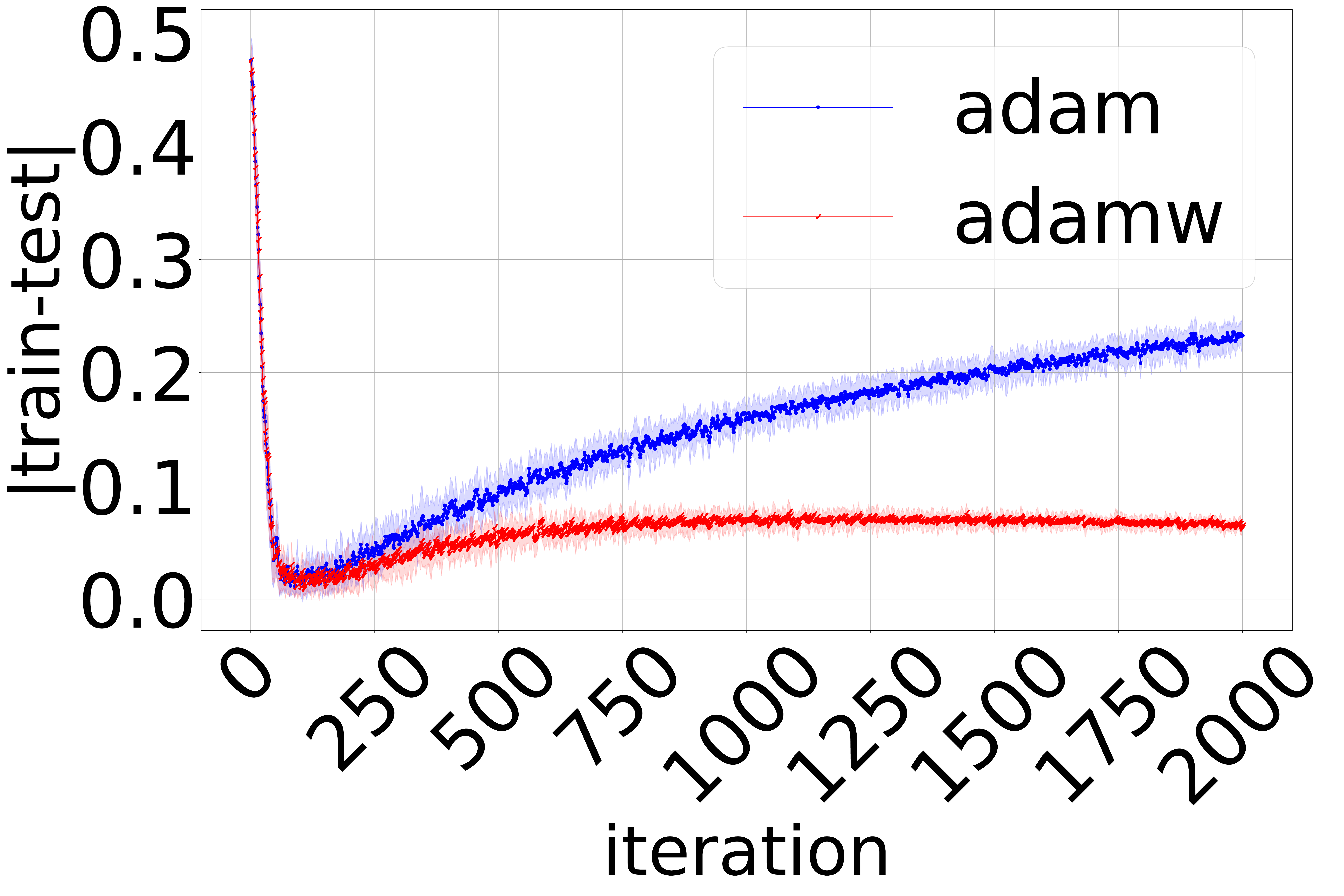}  
  \label{fig:sub-second}
\end{subfigure}
\begin{subfigure}{.24\textwidth}
  \centering
  \includegraphics[width=0.89\linewidth]{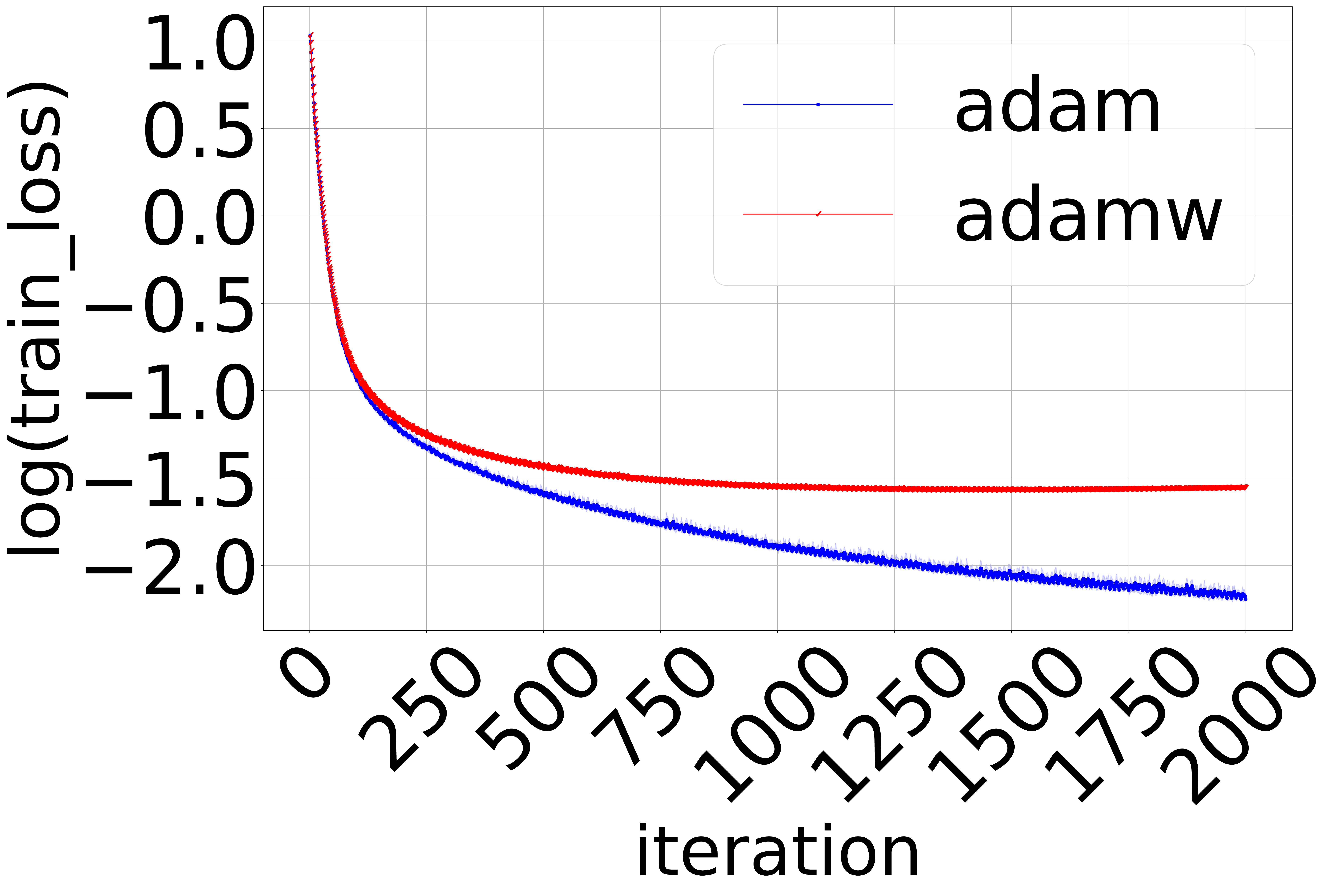}  
  \label{fig:sub-first}
\end{subfigure}
\begin{subfigure}{.24\textwidth}
  \centering
  \includegraphics[width=0.89\linewidth]{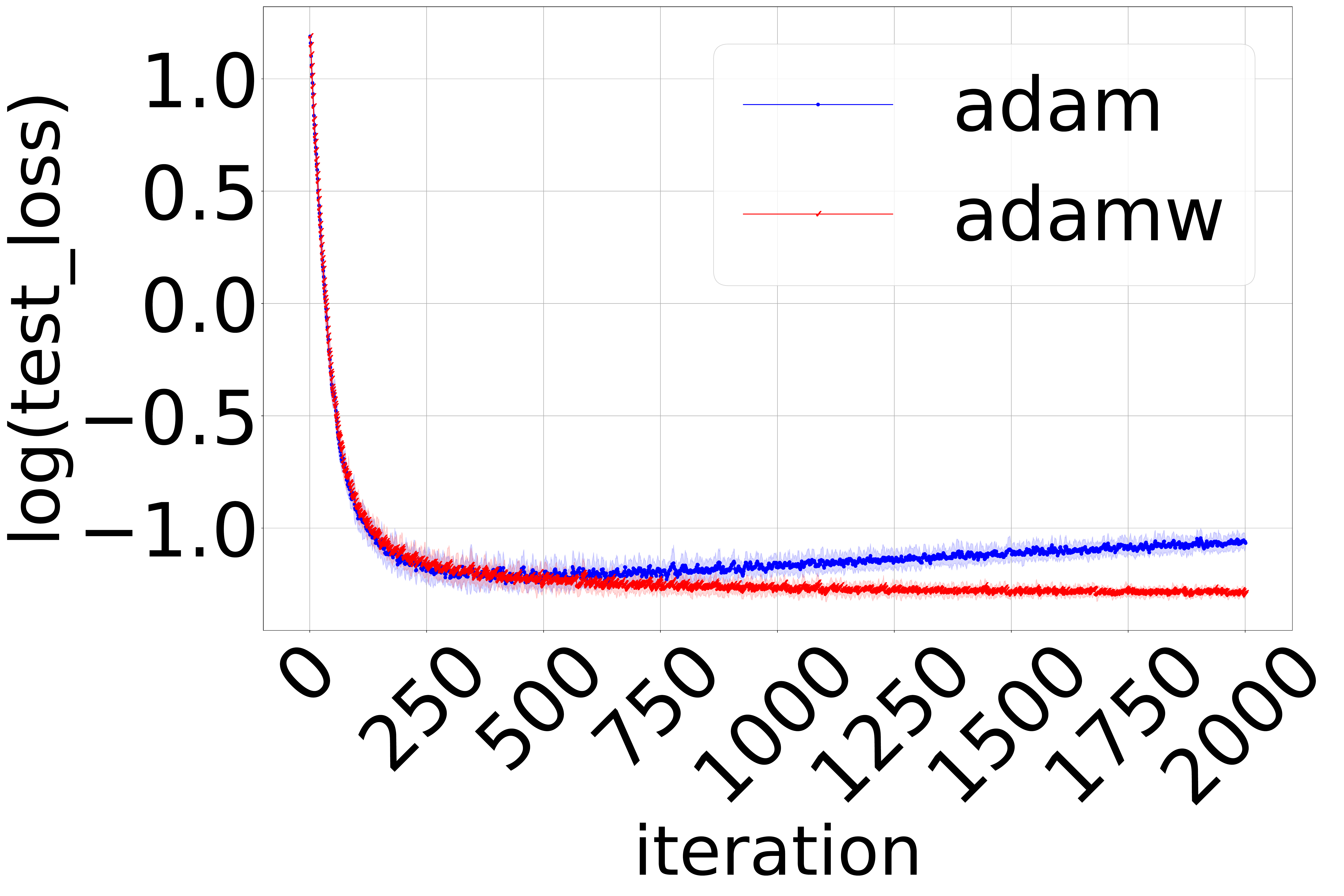}  
  \label{fig:sub-second}
\end{subfigure}
\\
\begin{subfigure}{.24\textwidth}
  \centering
  \includegraphics[width=0.89\linewidth]{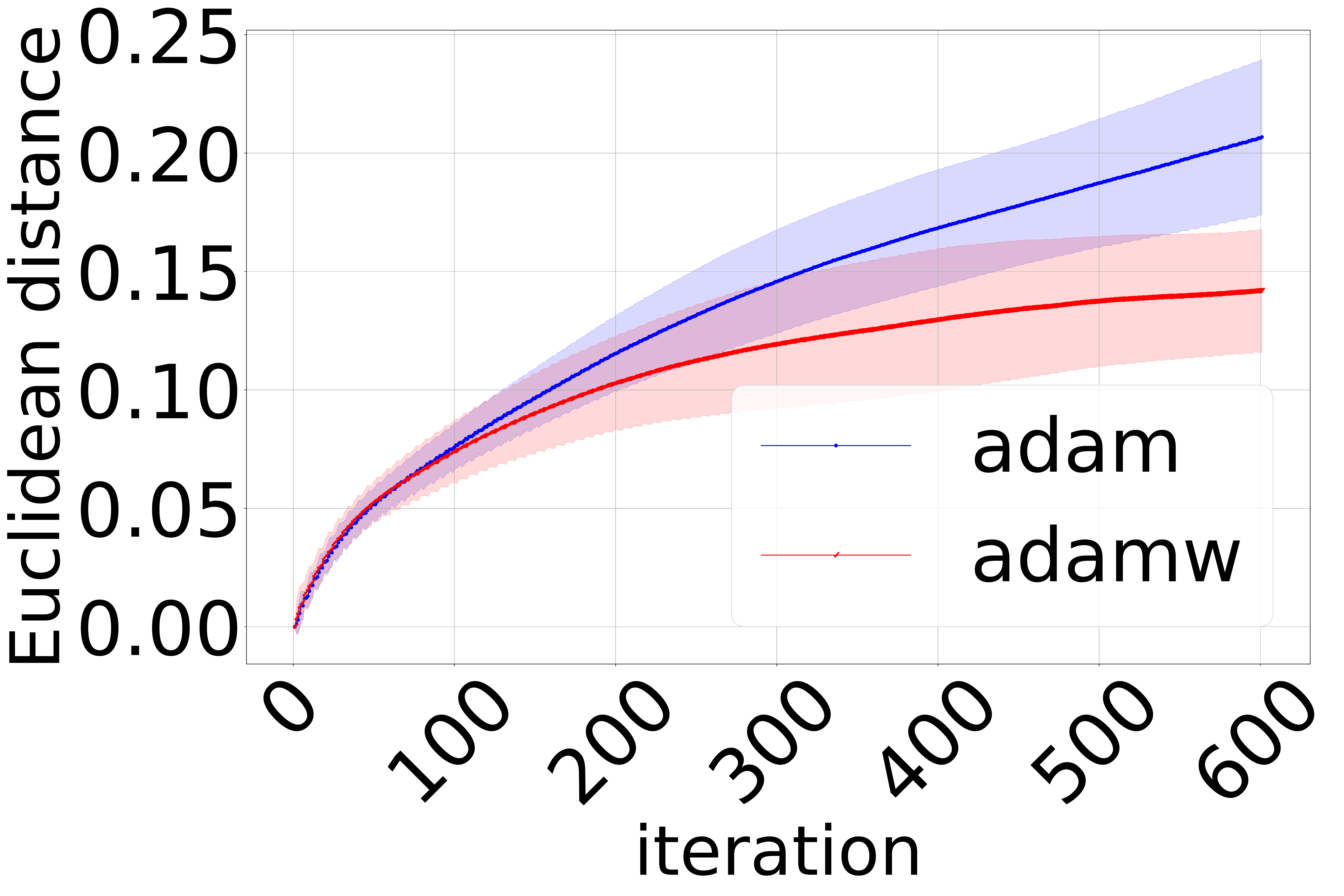}  
  \label{fig:sub-first}
\end{subfigure}
\begin{subfigure}{.24\textwidth}
  \centering
  \includegraphics[width=0.89\linewidth]{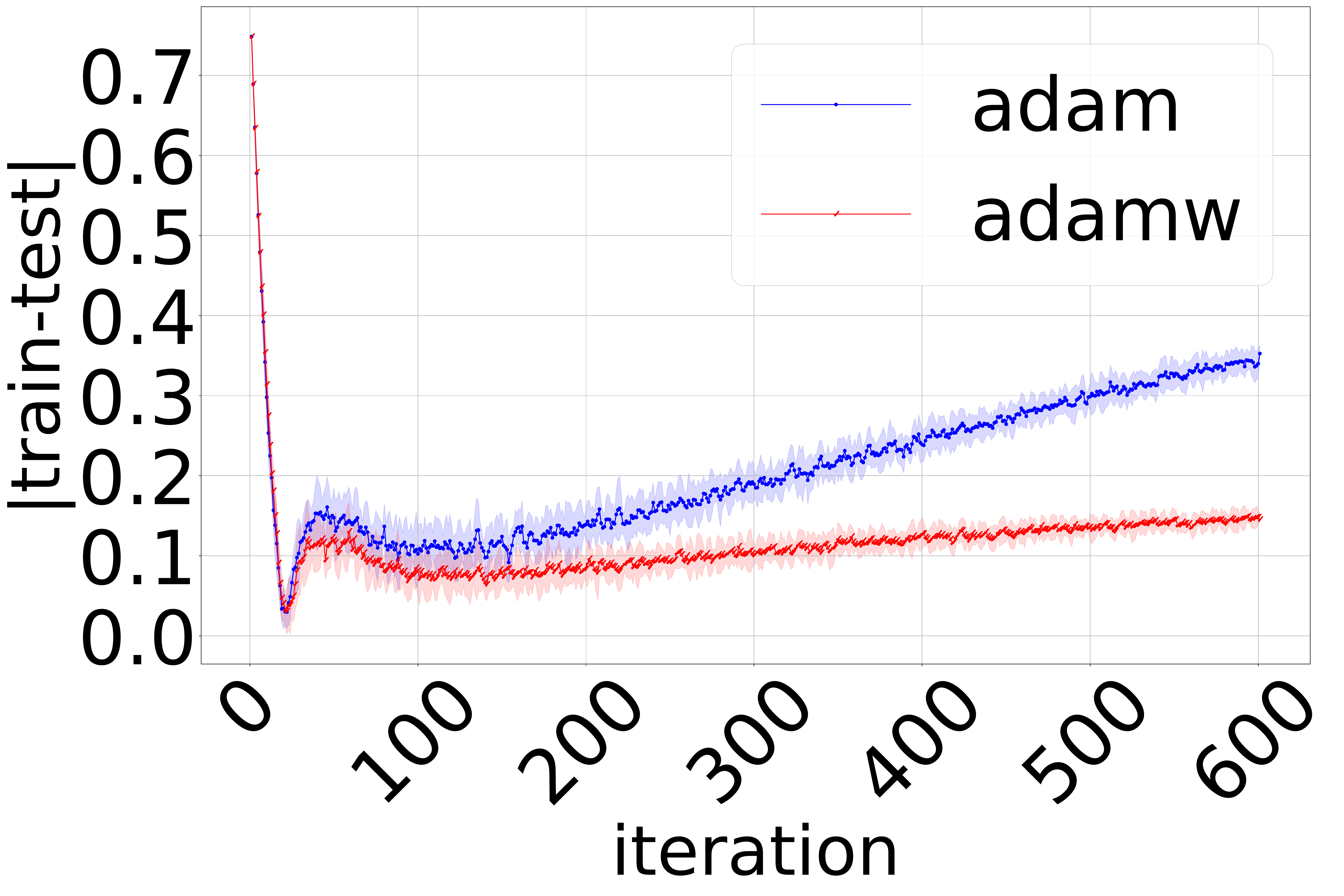}  
  \label{fig:sub-second}
\end{subfigure}
\begin{subfigure}{.24\textwidth}
  \centering
  \includegraphics[width=0.89\linewidth]{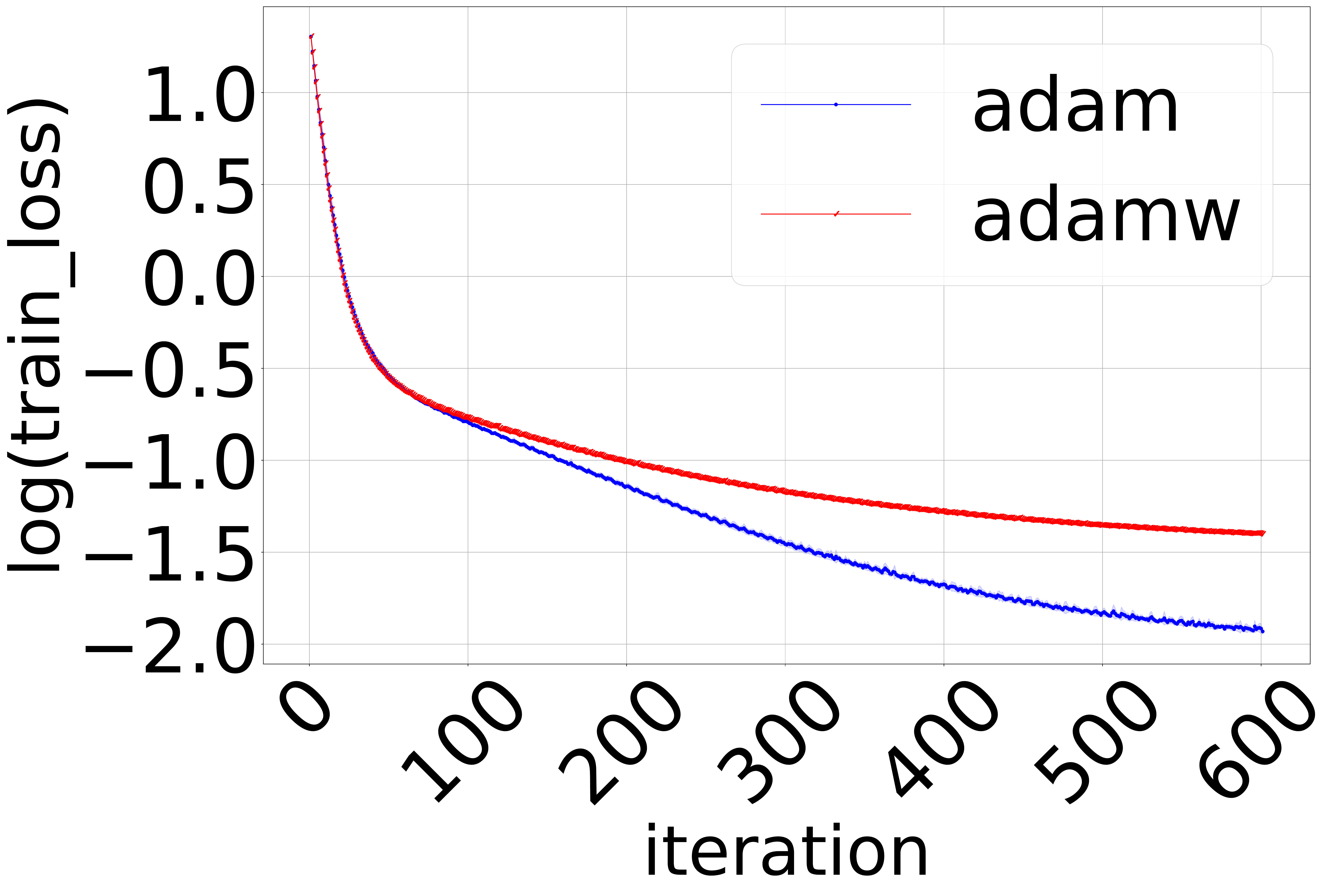}  
  \label{fig:sub-first}
\end{subfigure}
\begin{subfigure}{.24\textwidth}
  \centering
  \includegraphics[width=0.89\linewidth]{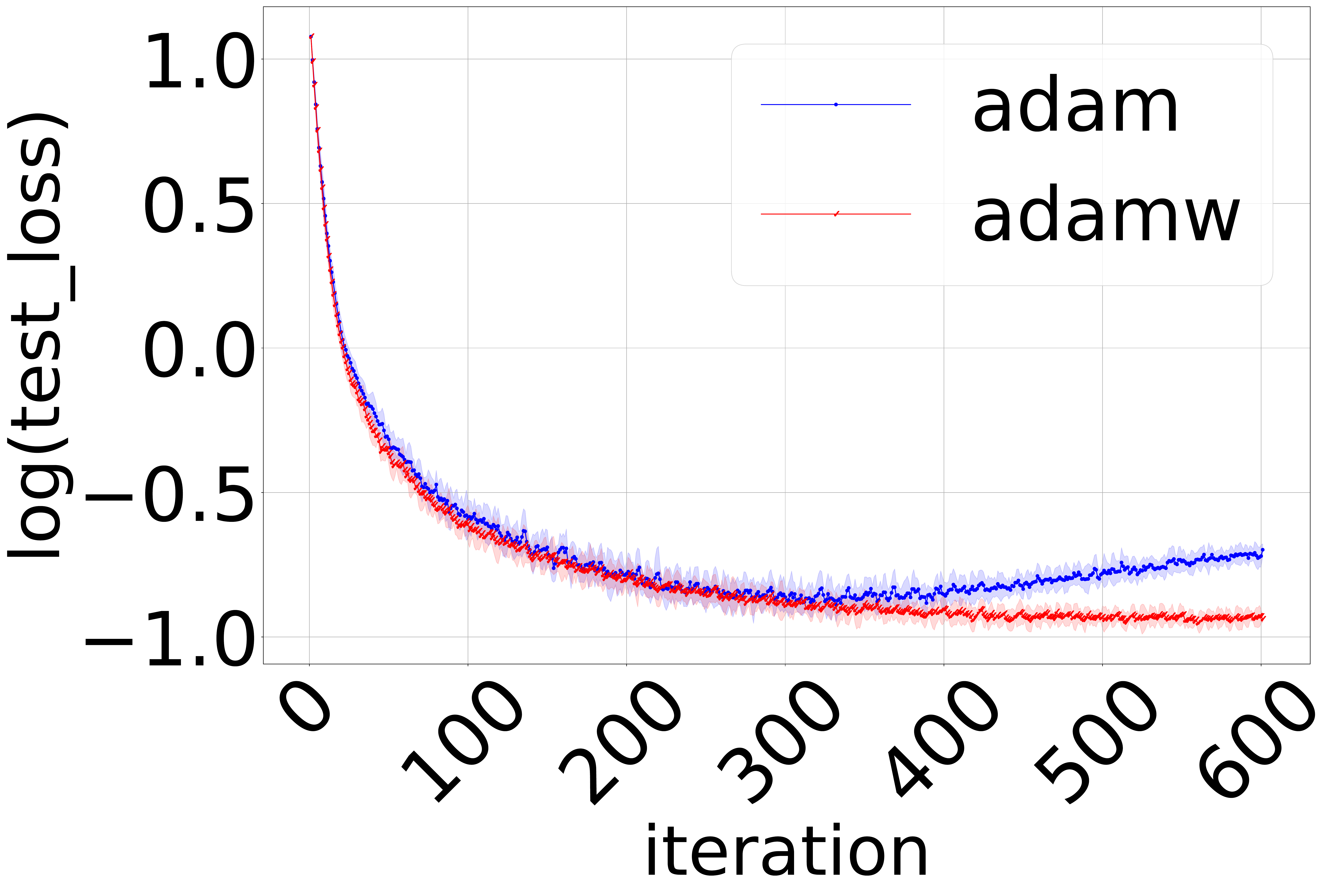}  
  \label{fig:sub-second}
\end{subfigure}
\caption{Performance comparison between Adam and AdamW when solving the CLS (top) and REG (bottom). Less overfitting happens to AdamW than to Adam. 
\textit{Left to right}: parameter distance, generalization error, train loss and test loss. The losses are plotted in log space. } 
\label{fig:adamvsadamw}
\vspace*{-4mm}
\end{figure*}

As we can see from Figure~\ref{fig:adamvsadamw}, AdamW has lower parameter distances and generalization errors comparing to those of Adam in both classification and regression tasks. From the loss patterns, we can also see that the model trained with Adam overfits the train data faster than the model trained with AdamW. These observations show that weight decay helps improve the stability and generalization of Adam, which validates the theory in Section~\ref{sec:stability}.

\subsection{Real data}
We solve the image classification task on Cifar10~\citep{krizhevsky2009learning} dataset that has 50,000 train and 10,000 test color images of the same resolution 3$\times$32$\times$32. In terms of data augmentation, we only apply mean normalization for both train and test data. For model architecture, we use VGG11~\citep{simonyan2014very}, of which the weights are initialized the same for all configurations and later trained with 60 epochs. Due to space limitation, additional information is moved to Appendix~\ref{appx:real}. 

\begin{figure*}[ht]
\vspace*{-2mm}
\begin{subfigure}{.19\textwidth}
  \centering
  \hspace*{-1cm}
  \includegraphics[width=\linewidth]{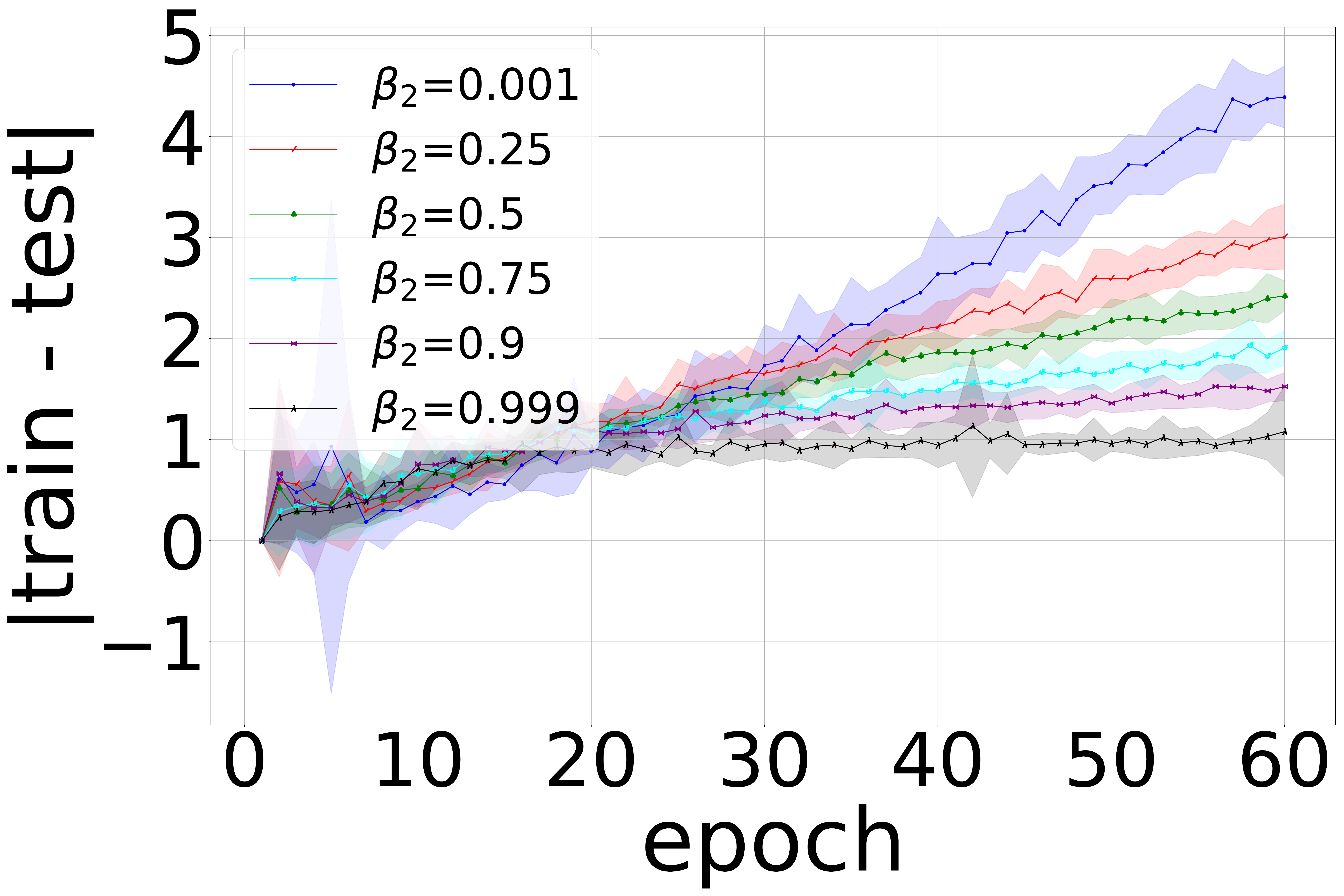}  
  \label{fig:sub-first}
\end{subfigure}
\begin{subfigure}{.2\textwidth}
  \centering
  \hspace*{-0.7cm}
  \includegraphics[width=\linewidth]{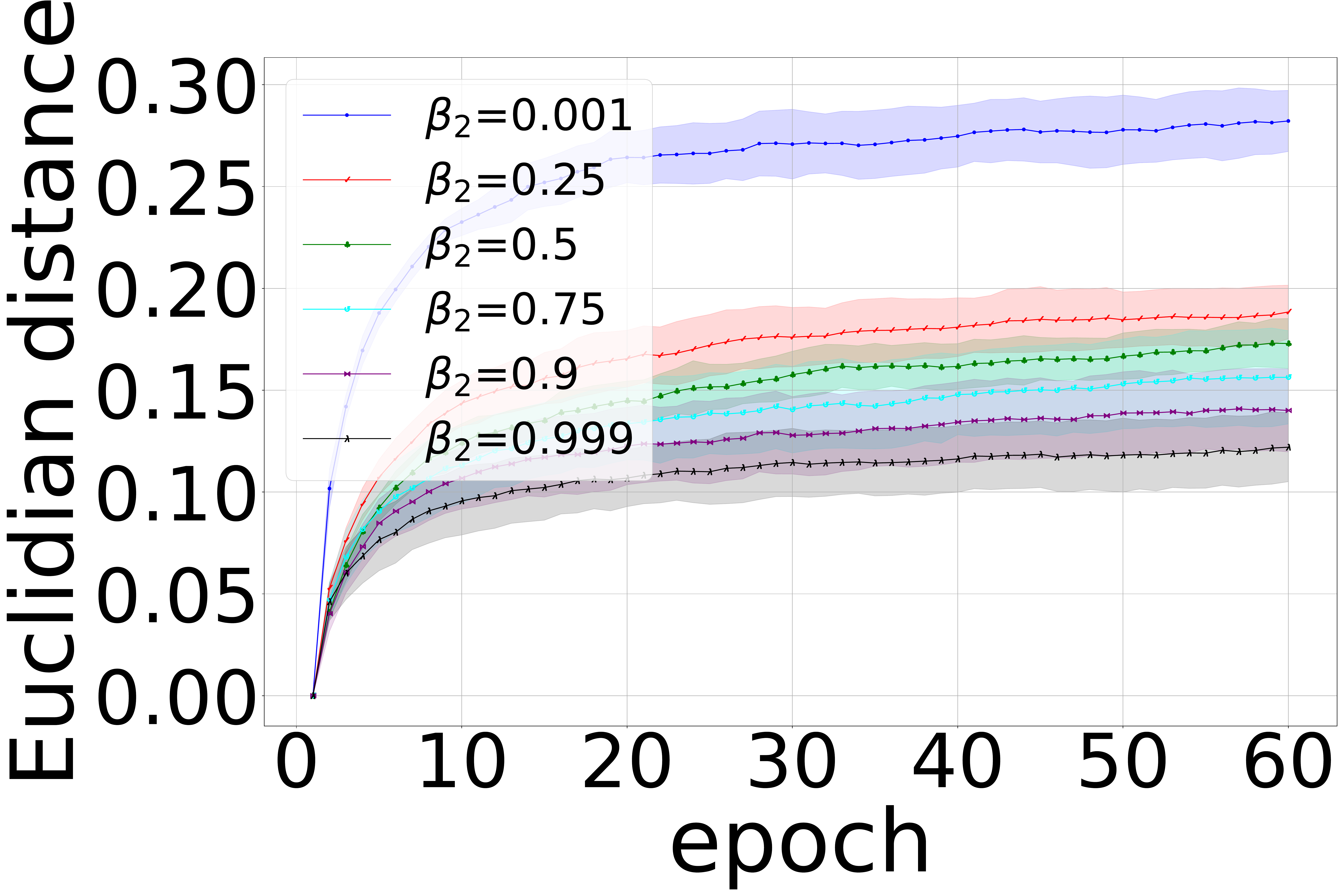}  
  \label{fig:sub-second}
\end{subfigure}
\begin{subfigure}{.2\textwidth}
  \centering
  \hspace*{-0.45cm}
  \includegraphics[width=\linewidth]{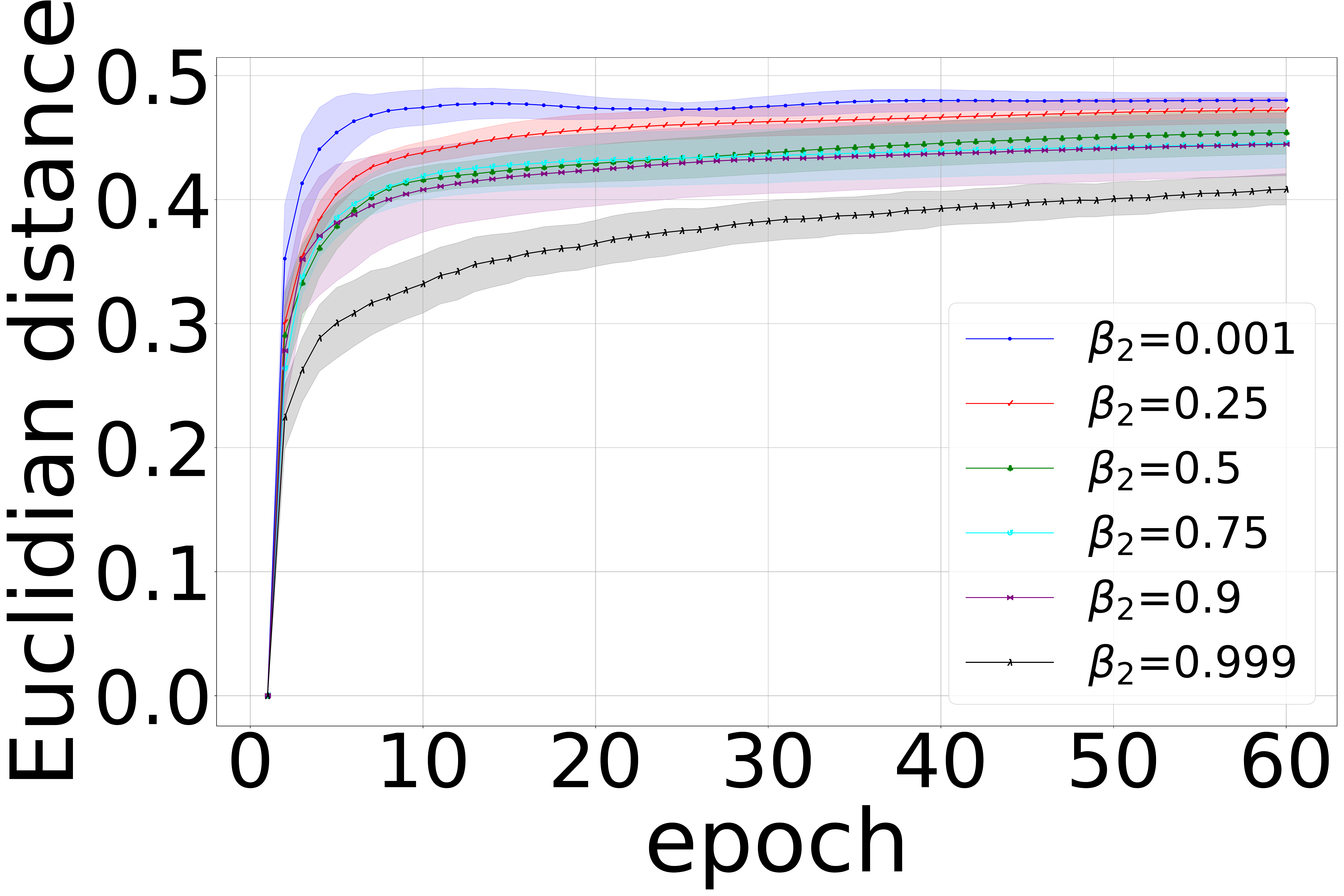}  
  \hspace*{-0.45cm}
  \label{fig:sub-second}
\end{subfigure}
\begin{subfigure}{.19\textwidth}
  \centering
  \includegraphics[width=0.95\linewidth]{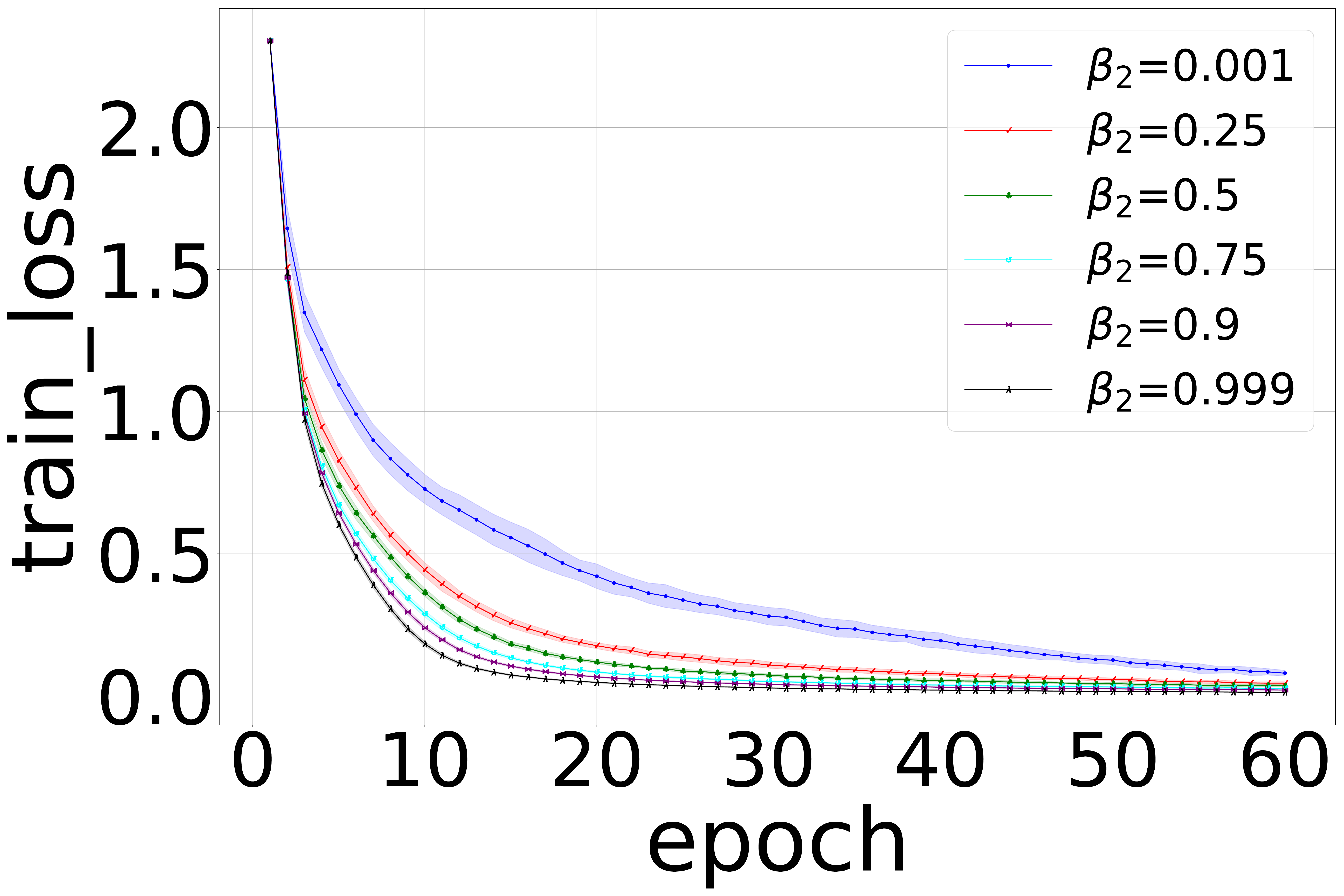}
    \hspace*{-0.7cm}
  \label{fig:sub-first}
\end{subfigure}
\begin{subfigure}{.19\textwidth}
  \centering
  \includegraphics[width=0.95\linewidth]{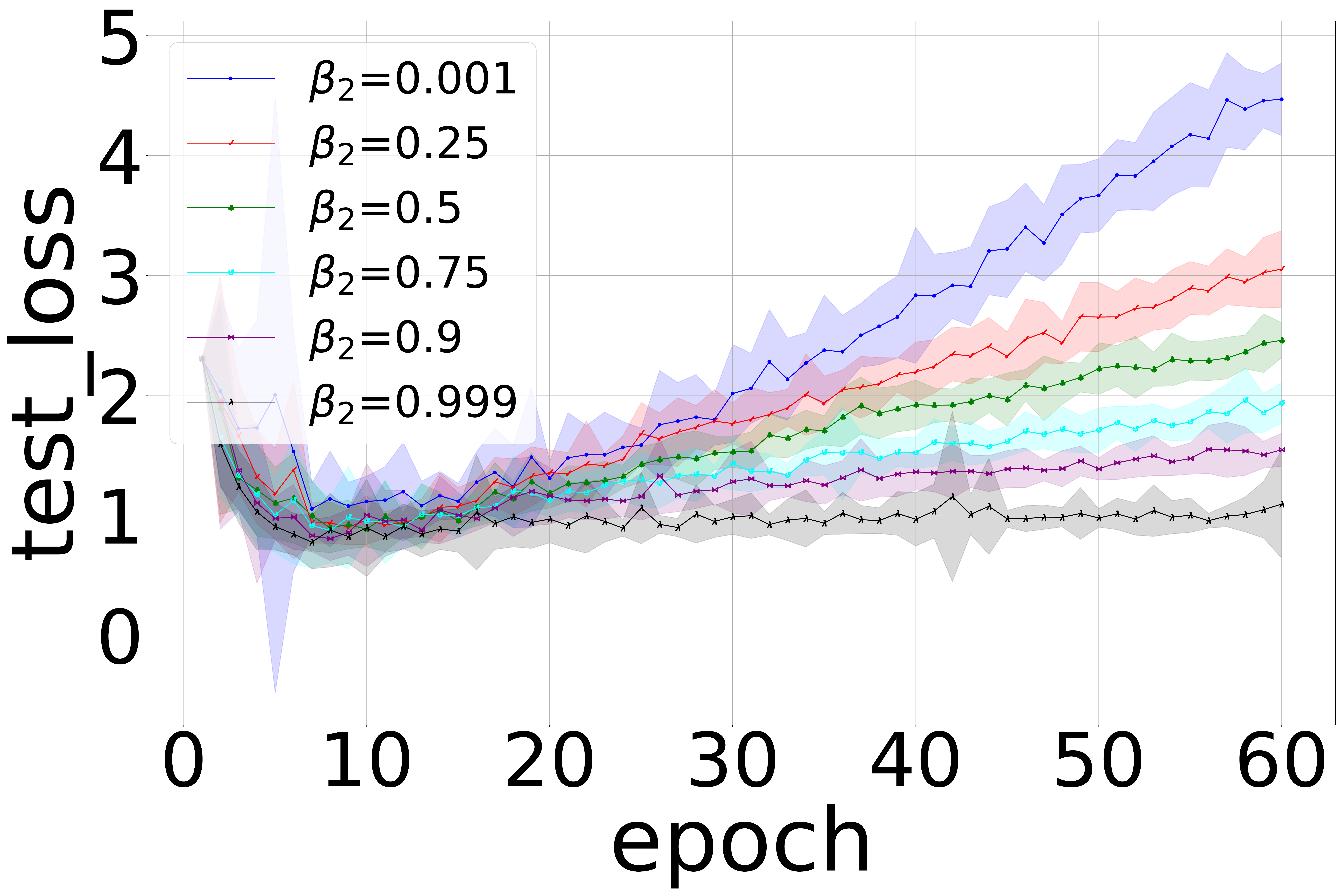}  
  \hspace*{-1cm}
  \label{fig:sub-second}
\end{subfigure}
\caption{Performances of models on Cifar10 dataset using VGG11. The models get better as $\beta_2$ increases from 0 to 1. 
\textit{Left to right}: generalization error (loss), parameter distance for the first convolution layer, the last fully-connected layer, train loss, and test loss.} 
\label{fig:cifar10}
\vspace*{-2mm}
\end{figure*}

As seen in Figure~\ref{fig:cifar10}, the results for Cifar10 agree with that of synthetic data. In detail, the more $\beta_2$ increases, the better the model is in terms of all metrics measured: parameter distance, generalization error, and both losses. In summary, both empirical experiments with real datasets support our theoretical results, in that the bigger $\beta_2$ is, the more stable and generalizable AOMs can get. 

\section{Conclusion}
This work establishes a novel theoretical result for stability and generalization for adaptive optimizers, which have been used intensively for, e.g., training neural networks. We show that AOMs depend heavily on a single $\beta_2$ value, providing helpful hints in tuning the training process, which is usually dependent on many hyperparameters. Our empirical experiments, which consider the applications of classification and regression, and employ synthetic to real-world datasets, reflect the results claimed in our theory. 
By building this theoretical framework with empirical validation, we hope to stimulate more work in this optimization area to help thoroughly answer the important question of choices: given a specific case, which optimizer should we use? 

\bibliography{main}
\bibliographystyle{unsrt}

\clearpage
\appendix

\section{Proof of Lemma~\ref{lem:decomposition}} \label{proof:decomposition}

In our analysis, $x_S$ is the output of an optimization algorithm at a particular time $T$ when the algorithm is used to minimized the empirical risk $F_S(x)$. For convenience, we also denote $x^*_S$ an empirical risk minimizer, i.e. $F_S(x^*_S) = \inf_{x\in\Omega} F_S(x)$. Note that $x_S$ and $x^*_S$ are in general not the same estimator, e.g, in deep learning we can only find local minimums due to non-convexity of the loss landscapes and early stopping is usually employed to terminate optimizers before they reach local minimums.  For simplicity, we assume that there is some $x^*\in\Omega$ such that $F(x^*) = \inf_{x\in\Omega}F(x)$. 

Recall that $\mathcal{R}(x_S) = F(x_S) - \inf_{x\in\Omega}F(x)$.  Given this set up, the excess risk decomposition is usually done in the following manner:
\begin{align}
  \mathcal{R}(x_S) = \underbrace{F(x_S) - F_S(x_S)}_{T_1} + \underbrace{F_S(x_S) - F_S(x^*)}_{T_2} +  
  \underbrace{F_S(x^*) -F(x^*)}_{T_3}.  
  \label{eq:risk}
\end{align}

Term $T_1$ is the generalization of error of the model $x_S$. Term $T_2$ is the empirical risk difference between the model $x_S$ and the population risk minimizer $x^*$. Term $T_3$ is the generalization error of $x^*$. By taking the expectation of (\ref{eq:risk}) with respect to the training set $S$ and notice that $\mathbb{E}[F_S(x^*) - F(x^*)] = 0$, we arrive at 
\begin{align}
 \mathbb{E}_S[\mathcal{R}(x_S)]\nonumber & = \mathbb{E}_S[F(x_S) - F_{S}(x_S)]
+  \mathbb{E}_S[F_S(x_S) - F_S(x^*)] \nonumber\\
& =  \mathbb{E}_S[F(x_S) - F_{S}(x_S)] 
+\mathbb{E}_S[F_S(x_S) - F_S(x_S^*)] 
+ \underbrace{\mathbb{E}_S[F_S(x_S^*) - F_S(x^*) ]}_{\le 0}\nonumber\\
& \le  \underbrace{\mathbb{E}_S[F(x_S) - F_{S}(x_S)]}_{\mathcal{E}_{gen}}
+ \underbrace{\mathbb{E}_S[F_S(x_S) - F_S(x_S^*)]}_{\mathcal{E}_{opt}},
\end{align}

where $\mathcal{E}_{gen}$ and $\mathcal{E}_{opt}$ are the expected generalization error and the expected optimization error respectively.

\section{Proof of Lemma~\ref{lem:stability}} 
\label{proof:stability}
Let $S$ and $S'$
 be two samples of size $n$ differing in only a single example, and let $z \in Z$ be an arbitrary example.  Consider running AOM on sample $S$ and $S'$, respectively. Let $A = \{\delta_{t_0} = 0\}$,  we then have that 
\begin{align*}
    \mathbb{E}|f(x_T;z) - f(x_T';z)|& = \mathbb{P}(A)\mathbb{E}[|f(x_T;z) - f(x_T';z)||A] + \mathbb{P}(A^c)\mathbb{E}[|f(x_T;z) - f(x_T';z)||A^c]\\
    &\le \mu\mathbb{E}[\|x_T - x_T'\||A] + 2\mathbb{P}(A^c)\sup_{x,z}|f(x;z)|\\
    &\le  \mu\mathbb{E}[\|x_T - x_T'\||A] + 2M\mathbb{P}(A^c).
\end{align*}
Let $i^*$ be the position where $S$ and $S'$ are different and denote the first time AOM uses the example $z_{i^*}$ by the random variable $I$. Since $\{I > i_0\}\subset\{\delta_{i_0} = 0\}$, we have that 
$$\mathbb{P}(A^c) = \mathbb{P}(\{\delta_{t_0} \not = 0\})\le \mathbb{P}(\{I \le t_0\})\le \frac{t_0}{n}.$$
\section{Proof of Theorem \ref{theorem2}} \label{appx:proof_theorem2}
\begin{proof}
For any $t > t_0$,  we have that 
\begin{align*}
\delta_t & = \left\|x_t - \eta_t\frac{\nabla f(x_t, z_{i_t})}{\sqrt{v_{t+1}+\epsilon}}- x_t' + \eta_t\frac{\nabla f(x_t', z_{i_t}')}{\sqrt{v'_{t+1}+\epsilon}}\right\|_2\\
& \le \|x_t - x_t'\|_2 + \eta_t\left\|\frac{\nabla f(x_t, z_{i_t})}{\sqrt{v_{t+1}+\epsilon}} - \frac{\nabla f(x_t', z_{i_t}')}{\sqrt{v'_{t+1}+\epsilon}}\right\|_2\\
& = \|x_t - x_t'\|_2 + \eta_t R_t,
\end{align*}
where we define 
$$R_t = \left\|\frac{\nabla f(x_t, z_{i_t})}{\sqrt{v_{t+1}+\epsilon}} - \frac{\nabla f(x_t', z_{i_t}')}{\sqrt{v'_{t+1}+\epsilon}}\right\|_2.$$
We consider two cases which depends on the realization of the selected index $i_t$.

\noindent\textbf{Case 1:} With probability $1-\frac{1}{n}$, we have $z_{i_t} = z'_{i_t}$. Thus, 
\begin{align*}
R_t &:= \left\|\frac{\nabla f(x_t, z_{i_t})}{\sqrt{v_{t+1}+\epsilon}} - \frac{\nabla f(x_t', z_{i_t})}{\sqrt{v'_{t+1}+\epsilon}}\right\|_2\\
& \le \left\|\frac{\nabla f(x_t, z_{i_t})}{\sqrt{v_{t+1}+\epsilon}} - \frac{\nabla f(x_t', z_{i_t})}{\sqrt{v_{t+1}+\epsilon}}\right\|_2
+ \left\|\frac{\nabla f(x_t'; z_{i_t})}{\sqrt{v_{t+1}+\epsilon}} - \frac{\nabla f(x_t'; z_{i_t})}{\sqrt{v'_{t+1}+\epsilon}}\right\|_2\\
& = I_t + J_t,
\end{align*}
where we denote
$$I_t:= \left\|\frac{\nabla f(x_t, z_{i_t})}{\sqrt{v_{t+1}+\epsilon}} - \frac{\nabla f(x'_t, z_{i_t})}{\sqrt{v_{t+1}+\epsilon}}\right\|_2,$$
$$J_t := \left\|\frac{\nabla f(x_t';z_{i_t})}{\sqrt{v_{t+1}+\epsilon}}-\frac{\nabla f(x_t';z_{i_t})}{\sqrt{v'_{t+1}+\epsilon}}\right\|_2.$$
We now derive bounds for both $I_t$ and $J_t$. We have 
\begin{align*}
I_t & = \left[ \sum_{j=1}^d\left(\frac{\nabla_j f(x_t,z_{i_t})}{\sqrt{v_{t+1,j}+\epsilon}} - \frac{\nabla_j f(x_t'; z_{i_t})}{\sqrt{v_{t+1,j}+\epsilon}}\right)^2\right]^{1/2}    \\
& = \left[ \sum_{j=1}^d\left(\frac{1}{v_{t+1,j}+\epsilon}\right)(\nabla_jf(x_t, z_{i_t})- \nabla_j f(x'_t; z_{i_t}))^2\right]^{1/2} \\
& \le  \left[ \sum_{j=1}^d\left(\frac{1}{\lambda_1+\epsilon}\right)(\nabla_jf(x_t, z_{i_t})- \nabla_j f(x'_t; z_{i_t}))^2\right]^{1/2} \\
& = \frac{1}{\sqrt{\lambda_1+\epsilon}}\|\nabla f(x_t, z_{i_t}) - \nabla f(x_t', z_{i_t})\|_2\\
&\le \frac{L}{\sqrt{\lambda_1+\epsilon}}\|x_t - x_t'\|_2,
\end{align*}
where the first and last  inequalities follows from assumptions (3) and (1) respectively.
\begin{align*}
J_t & = \left[\sum_{j=1}^d \left(\frac{\nabla_j f(x_t', z_{i_t})}{\sqrt{v_{t+1,j}+\epsilon}} - \frac{\nabla_j f(x'_t; z_{i_t})}{\sqrt{v'_{t+1,j}+\epsilon}}\right)^2\right]^{1/2} \\
& = \left[\sum_{i=1}^d\nabla_j f(x_t', z_{i_t})^2\left(\frac{1}{\sqrt{v_{t+1,i}+\epsilon}}- \frac{1}{\sqrt{v'_{t+1,j}+\epsilon}}\right)^2\right]^{1/2}\\
& \le  \left[\sum_{i=1}^d\mu^2\left(\frac{1}{\sqrt{v_{t+1,i}+\epsilon}}- \frac{1}{\sqrt{v'_{t+1,j}+\epsilon}}\right)^2\right]^{1/2}\\
& = \mu\left\|\frac{1}{\sqrt{v_{t+1}+\epsilon}}- \frac{1}{\sqrt{v'_{t+1}+\epsilon}}\right\|_2,
\end{align*}
where the  inequality is from assumption (1).

By Lemma~\ref{lemma4}, we deduce that  
\begin{align*}
R_t & \le  \frac{L}{\sqrt{\lambda_1+\epsilon}}\|x_t - x_t'\|_2 +  \mu\left\|\frac{1}{\sqrt{v_{t+1}+\epsilon}}- \frac{1}{\sqrt{v'_{t+1}+\epsilon}}\right\|_2\\
&\le \frac{L}{\sqrt{\lambda_1+\epsilon}}\|x_t - x_t'\|_2 + \frac{\mu}{2\sqrt{\lambda_1+\epsilon}(\lambda_1+\epsilon)}\|v_{t+1} - v'_{t+1}\|_2.
\end{align*}
Since in this case $i_t = i'_t$, we further have that 

\begin{align*}
 &\|v_{t+1} - v'_{t+1}\|_2 \\ & =  \|\beta_t v_t + (1-\beta_t)\nabla f(x_t, z_{i_t})^2  - \beta_t v'_t 
- (1-\beta_t)\nabla f(x'_t, z_{i_t})^2\|_2\\
& \le \beta_t\|v_t - v_t'\|_2 + (1-\beta_t)\|\nabla f(x_t, z_{i_t})^2 - \nabla f(x'_t, z_{i_t})^2\|_2 \\
& = \beta_t\|v_t - v_t'\|_2 + (1- \beta_t)\left[\sum_{j=1}^d\left(\nabla_jf(x_t, z_{i_t})^2 - \nabla_j f(x'_t, z_{i_t})^2\right)^2\right]^{1/2}\\
& = \beta_t\|v_t - v_t'\|_2  + (1-\beta_t)
\left[
    \sum_{j = 1}^d (\nabla_jf(x_t, z_{i_t})+  
    \nabla_j f(x'_t,z_{i_t}))^2(\nabla_jf(x_t, z_{i_t})-\nabla_j f(x'_t,z_{i_t}))^2
\right]^{1/2}
\\
& \le \beta_t\|v_t - v_t'\|_2 +  (1-\beta_t)2\mu\left[\sum_{j = 1}^d(\nabla_jf(x_t, z_{i_t})-\nabla_j f(x'_t,z_{i_t}))^2\right]^{1/2}\\
& = \beta_t\|v_t - v_t'\|_2 + 2\mu(1-\beta_t)\|\nabla f(x_t, z_{i_t}) - \nabla f(x'_t, z_{i_t})\|_2\\
& \le \beta_t\|v_t - v_t'\|_2 + 2\mu L(1-\beta_t)\|x_t - x_t'\|_2,
\end{align*}

where the second inequality follows from the fact that $(a+b)^2\le 2(a^2+b^2)$ and the assumption (1).
Thus, we deduce that 
\begin{align*}
R_t& \le \frac{L}{\sqrt{\lambda_1+\epsilon}}\|x_t - x_t'\|_2 
+\frac{\mu}{2\sqrt{\lambda_1+\epsilon}(\lambda_1+\epsilon)}(\beta_t\|v_t - v_t'\|_2 
+ 2\mu L(1-\beta_t)\|x_t - x_t'\|_2)    \\
& = \frac{L}{\sqrt{\lambda_1+\epsilon}}\|x_t - x_t'\|_2 + \frac{\mu\beta_t}{2\sqrt{\lambda_1+\epsilon}(\lambda_1+\epsilon)}\|v_t - v'_t\|_2 
+ \frac{\mu^2L(1-\beta_t)}{\sqrt{\lambda_1+\epsilon}(\lambda_1+\epsilon)}\|x_t - x_t'\|_2\\
& = \frac{L}{\sqrt{\lambda_1+\epsilon}}\left[1 + \frac{\mu^2 (1-\beta_t)}{\lambda_1+\epsilon}\right]\|x_t - x'_t\|_2  + \frac{\mu\beta_t}{2\sqrt{\lambda_1+\epsilon}(\lambda_1+\epsilon)}\|v_t - v_t'\|_2.
\end{align*}
Thus with probability $1-\frac{1}{n}$, we obtain
\begin{align*}
   {R_t} & {\le \frac{L}{\sqrt{\lambda_1+\epsilon}}\left[1 +  \frac{\mu^2 (1-\beta_t)}{\lambda_1+\epsilon}\right]\|x_t - x'_t\|_2  } \\ & 
   \hspace{2cm}+
    {\frac{\mu\beta_t}{2\sqrt{\lambda_1+\epsilon}(\lambda_1+\epsilon)}\|v_t - v_t'\|_2}.
\end{align*}


\textbf{Case 2:} With probability $\frac{1}{n}$, we have $z_{i_t} \not= z'_{i_t}$. Thus,
\begin{align*}
R_t & \le \left\|\frac{\nabla f(x_t, z_{i_t})}{\sqrt{v_{t+1}+\epsilon}}\right\|_2 + \left\|\frac{\nabla f(x'_t, z'_{i_t})}{\sqrt{v_{t+1}+\epsilon}}\right\|_2 \\
&\le \left[\sum_{j=1}^d\frac{\nabla_j f(x_t, z_{i_t})^2}{v_{t+1,j}+\epsilon}\right]^{1/2} + \left[\sum_{j=1}^d\frac{\nabla_j f(x'_t, z'_{i_t})^2}{v'_{t+1,j}+\epsilon}\right]^{1/2}\\
&\le \frac{[\sum_{j=1}^d\nabla_j f(x_t,z_{i_t})^2]^{1/2}}{\sqrt{\lambda_1 + \epsilon}} +  \frac{[\sum_{j=1}^d\nabla_j f(x'_t,z'_{i_t})^2]^{1/2}}{\sqrt{\lambda_1 + \epsilon}}\\
& = \frac{\|\nabla f(x_t, z_{i_t})\|_2}{\sqrt{\lambda_1 + \epsilon}} + \frac{\|\nabla f(x'_t, z'_{i_t})\|_2}{\sqrt{\lambda_1 + \epsilon}}\\
& \le \frac{2}{\sqrt{\lambda_1+\epsilon}}\mu,
\end{align*}
where the second inequality follows from assumption (1).

Thus with probability $\frac{1}{n}$, we have 
$${R_t \le \frac{2}{\sqrt{\lambda_1+\epsilon}}\mu}.$$
Thus, we obtain
\begin{align*}
    \delta_{t+1} & \le \delta_t  + \eta_t \left(1-\frac{1}{n}\right)\frac{L}{\sqrt{\lambda_1+\epsilon}}\left[1 + \frac{\mu^2 (1-\beta_t)}{\lambda_1+\epsilon}\right]\delta_t\\
    &\hspace{1cm} + \eta_t \left(1-\frac{1}{n}\right)\frac{\mu\beta_t}{2\sqrt{\lambda_1+\epsilon}(\lambda_1+\epsilon)}\sigma_t + \eta_t\frac{1}{n}\frac{2}{\sqrt{\lambda_1+\epsilon}}\mu.
\end{align*}
By taking the expectation conditioned over $\delta_{t_0} = 0$ of the above inequality, we get the conclusion of the theorem.
\end{proof}

\section{Proof of Theorem \ref{theorem3}} \label{appx:proof_theorem3}
\begin{proof}
For any $t > t_0$, we have that 
\begin{align*}
\|v_{t+1} - v'_{t+1}\|_2 &\le \|\beta_t v_t +(1-\beta_t)\nabla f(x_t; z_{i_t})^2  - \beta_t v'_t - (1-\beta_t)\nabla f(x'_t; z'_{i_t})^2\|_2 \\
& \le \beta_t\|v_t - v'_t\|_2 + (1-\beta_t)\|\nabla f(x_t; z_{i_t})^2 - \nabla f(x'_t; z'_{i_t})^2\|_2\\
& =  \beta_t\|v_t - v'_t\|_2 + (1-\beta_t)N_t,
\end{align*}
where we define 
$$N_t = \|\nabla f(x_t; z_{i_t})^2 - \nabla f(x'_t; z'_{i_t})^2\|_2.$$
We again consider two cases which depends on the realization of the selected index $i_t$.\\
\textbf{Case 1:} With probability $1-\frac{1}{n}$, we have  $z_{i_t} = z'_{i_t}$. Thus,
\begin{align*}
N_t & = \|\nabla f(x_t, z_{i_t})^2 - \nabla f(x'_t, z_{i_t})^2\|_2\\
& = \left[\sum_{j=1}^d\left(\nabla_jf(x_t, z_{i_t})^2 - \nabla_j f(x'_t, z_{i_t})^2\right)^2\right]^{1/2}\\
& =\Bigg [\sum_{j = 1}^d (\nabla_jf(x_t, z_{i_t})+ \nabla_j f(x'_t,z_{i_t}))^2(\nabla_jf(x_t, z_{i_t}) - \nabla_j f(x'_t,z_{i_t}))^2\Bigg]^{1/2}\\
& \le 2\mu\left[\sum_{j = 1}^d(\nabla_jf(x_t, z_{i_t})-\nabla_j f(x'_t,z_{i_t}))^2\right]^{1/2}\\
& =  2\mu\|\nabla f(x_t, z_{i_t}) - \nabla f(x'_t, z_{i_t})\|_2\\
& \le 2L\mu\|x_t - x_t'\|_2,
\end{align*}
where the first inequality follows from the fact that $(a+b)^2\le 2(a^2+b^2)$ and the assumption (1).

Thus with probability $1-\frac{1}{n}$, we have that  
$$N_t \le 2L\mu\delta_t.$$

\textbf{Case 2:} With probability $\frac{1}{n}$, we have $z_{i_t}\not= z'_{i_t}$. Thus
\begin{align*}
N_t & \le \|\nabla f(x_t, z_{i_t})^2\|_2 + \|\nabla f(x'_t, z'_{i_t})^2\|_2\\
& = \left[\sum_{j = 1}^d\nabla_j f(x_t, z_{i_t})^4\right]^{1/2} +  \left[\sum_{j = 1}^d\nabla_j f(x'_t, z'_{i_t})^4\right]^{1/2}\\
& \le \sum_{j =1}^d\nabla_jf(x_t, z_{i_t})^2 +  \sum_{j =1}^d\nabla_jf(x'_t, z'_{i_t})^2\\
& = \|\nabla f(x_t,z_{i_t})\|_2^2 +  \|\nabla f(x'_t,z'_{i_t})\|_2^2\\
& \le 2\mu^2,
\end{align*}
where the first inequality follows from the fact that for $a_1,\ldots,a_d \ge 0, (a_1+\ldots + a_d)^2 \ge a_1^2+\ldots+a_d^2$.

Thus with probability $\frac{1}{n}$, we have that 
$$N_t \le 2\mu^2.$$
By combining the above two cases, we obtain
\begin{align*}
\sigma_{t+1} & \le \beta_t \sigma_t +  (1- \beta_t)\left(1-\frac{1}{n}\right)2L\mu\delta_t + \frac{1}{n} (1-\beta_t)2\mu^2.   
\end{align*}
Taking the expectation conditioned over $\delta_{t_0} = 0$ of the above inequality, we get the conclusion of the theorem.
\end{proof}

\section{Proof of Lemma~\ref{normest}} \label{appx:proof_lemma3}

\begin{proof}
By substituting $\eta_t = \frac{c}{t}$ and $\beta_t = 1-\alpha_t$ into the matrix $A_t$, we have that 
\vspace{-2mm}
$$A_t = \begin{bmatrix}
1+ \frac{c}{t}U_t & \frac{c}{t} V_t \\
\alpha_t W & 1 - \alpha_t
\end{bmatrix}$$
\vspace{-2mm}

For any $v=(v_1, v_2)\in \mathbb{R}^2$ such that $v_1^2 + v_2^2 = 1$, we have that 
\begin{align*}
\|A_tv\|_2^2 
& = \begin{Vmatrix} v_1 + \frac{c}{t}U_t v_1 + \frac{c}{t} V_t v_2 \\ \alpha_tWv_1 + v_2 - \alpha_t v_2  \end{Vmatrix}^2_2   \\
& = \left(v_1 + \frac{c}{t}U_t v_1 + \frac{c}{t} V_t v_2\right)^2 + \left(\alpha_tWv_1 + v_2 - \alpha_t v_2\right)^2 \\
& = v_1^2 + 2 \frac{cv_1}{t}(U_tv_1 + V_tv_2) + \frac{c^2(U_tv_1 + V_tv_2)^2}{t^2}\\
&\hspace{2cm} + v_2^2 + 2v_2\alpha_t(Wv_1 - v_2) +\alpha_t^2(Wv_1 - v_2)^2\\
& \le 1 + \frac{2c}{t}\sqrt{U_t^2+V_t^2} + c^2\frac{U_t^2+V_t^2}{t^2} + 2\alpha_t\sqrt{W_t^2+1} + \alpha_t^2(W_t^2 + 1)\\
& \le 1 + \frac{2c}{t}\sqrt{U^2+V^2} + c^2\frac{U^2+V^2}{t^2} + 2\alpha_t\sqrt{W^2+1} + \alpha_t^2(W^2 + 1)\\
& \le 1 + \frac{2c}{t}\sqrt{U^2+V^2} + c^2\frac{U^2+V^2}{t^2} + \alpha_t\left(2\sqrt{W^2+1} + (W^2 + 1)\right)\\
& \le exp\Bigg[\frac{2c}{t}\sqrt{U^2+V^2} + c^2\frac{U^2+V^2}{t^2} +  \alpha_t\left(2\sqrt{W^2+1} + (W^2 + 1)\right)\Bigg].\\
\end{align*}

\vspace{-5mm}
Where we have used the fact that $|v_1|, |v_2| \le 1$, the Cauchy–Schwarz inequality and $1+x \le e^x$ for all $x>0$.

Thus we deduce that
\vspace{-1mm}
\begin{align*}
\|A_t\|_2 &\le 
\biggl\{
exp\Big(\frac{2c}{t}\sqrt{U^2+V^2} + c^2\frac{U^2+V^2}{t^2} + \alpha_t\left(2\sqrt{W^2+1} + (W^2 + 1)\right)\Big)
\biggl\}^{\!1/2}\\
& = exp\Bigg(\frac{c}{t}\sqrt{U^2+V^2} + c^2\frac{U^2+V^2}{2t^2}+ \alpha_t\left(\sqrt{W^2+1} + \frac{1}{2}(W^2 + 1)\right)\Bigg).
\end{align*}

Denote 
$D_1 = {U^2 + V^2}, D_2 = \sqrt{W^2+1} + \frac{1}{2}(W^2 + 1)$, 
we can then rewrite 
$$\|A_t\|_2 \le exp\left(\frac{c}{t}\sqrt{D_1} + \frac{c^2}{2t^2}D_1 + \alpha_t D_2 \right).$$
Thus, we obtain the desired conclusion of the lemma.
\end{proof}

\section{Proof of Theorem~\ref{theorem4}} \label{appx:proof_theorem4}

\begin{proof}
Denote 
$$H_t: = \left\|\begin{bmatrix} \Delta_t \\ \Sigma_t\end{bmatrix}\right\|_2$$
By taking the norm on both sides of the relation (\ref{maineq}), we obtain
\begin{align*}
H_{t+1} & \le \|A_t\|H_t + \frac{1}{n}\sqrt{\frac{c^2Y^2}{t^2} + \alpha_t^2 Z^2} \\
& \le  exp\left(\frac{c}{t}\sqrt{D_1} + \frac{c^2}{t^2}\frac{D_1}{2} + \alpha_t D_2 \right) H_t + \frac{1}{n} \left(\frac{c}{t}Y + \alpha_tZ\right).
\end{align*}
By expanding the above inequality, for any $T > t_0$ we have 
\begin{align*}
H_T & \le \sum_{t= t_0 +1}^T \left[\prod_{k = t+1}^{T}exp\left(\frac{c}{k}\sqrt{D_1} + \frac{c^2}{k^2}\frac{D_1}{2} + \alpha_k D_2 \right)\right]  \times \frac{1}{n}\left(\frac{c}{t}Y + \alpha_tZ\right) \\
& = \frac{c}{n}Y\sum_{t= t_0 +1}^T \left[\prod_{k = t+1}^{T}exp\left(\frac{c}{k}\sqrt{D_1} + \frac{c^2}{k^2}\frac{D_1}{2} + \alpha_k D_2 \right)\right]\frac{1}{t} \\ 
&\hspace{2cm} + \frac{1}{n}Z\sum_{t= t_0 +1}^T \left[\prod_{k = t+1}^{T}exp\left(\frac{c}{k}\sqrt{D_1} + \frac{c^2}{k^2}\frac{D_1}{2} + \alpha_k D_2 \right)\right]\alpha_t\\ 
& \le \frac{c}{n}Y\sum_{t=t_0 +1}^{T}exp\Bigg(c\sqrt{D_1}\sum_{k=t+1}^T\frac{1}{k}  + \frac{c^2D_1}{2}\sum_{k=t+1}^T\frac{1}{k^2} + D_2\sum_{k= t+1}^T\alpha_t\Bigg)\frac{1}{t}\\
&\hspace{2cm} + \frac{1}{n}Z\sum_{t = t_0 +1}^Texp\Bigg(c\sqrt{D_1}\sum_{k=t+1}^T\frac{1}{k} + \frac{c^2D_1}{2}\sum_{k=t+1}^T\frac{1}{k^2} + D_2\sum_{k= t+1}^T\alpha_t\Bigg)\alpha_t\\
& = \frac{c}{n}Y\sum_{t = t_0+1}^T exp\left(c\sqrt{D_1}\log\Bigg(\frac{T}{t}\right)+ \gamma\frac{c^2D_1}{2} +  D_2\sum_{k= t+1}^T\alpha_t\Bigg)\frac{1}{t}  \\
&\hspace{2cm} + \frac{1}{n}Z\sum_{t = t_0 +1}^Texp\left(c\sqrt{D_1}\log\Bigg(\frac{T}{t}\right)  + \gamma\frac{c^2D_1}{2} +  D_2\sum_{k= t+1}^T\alpha_t\Bigg)\alpha_t\\
&  = \frac{c}{n}Y exp\left(\gamma\frac{c^2D_1}{2}\right) \times \sum_{t= t_0+1}^Texp\left( D_2\sum_{k= t+1}^T\alpha_t\right) \left(\frac{T}{t}\right)^{c\sqrt{D_1}}\frac{1}{t}\\&\hspace{2cm} + \frac{1}{n}Z exp\left(\gamma\frac{c^2D_1}{2}\right) \times \sum_{t = t_0+1}^Texp\left( D_2\sum_{k= t+1}^T\alpha_t\right)\left(\frac{T}{t}\right)^{c\sqrt{D_1}}\alpha_t
\end{align*}
where we define 
$$\gamma = \sum_{k=1}^{\infty}\frac{1}{k^2}.$$
Thus, we obtain
\begin{align*}
H_t \le \frac{1}{n} &exp\left(\gamma\frac{c^2D_1}{2}\right)\times 
\sum_{t=t_0+1}^Texp\left( D_2\sum_{k= t+1}^T\alpha_t\right)\left(\frac{T}{t}\right)^{c\sqrt{D_1}} \times 
\left(\frac{1}{t}cY+ \alpha_t Z\right)
\end{align*}
which the conclusion of the theorem.
\end{proof}

\section{Proof of Corollary~\ref{corollary1}} \label{appx:proof_corollary1}
\begin{proof}
By the assumption, we have 
\begin{align*}
H_t & \le \frac{1}{n} exp\left(\gamma\frac{c^2D_1}{2}\right) 
\times \sum_{t=t_0+1}^Texp\left( D_2\alpha\sum_{k= t+1}^T\frac{1}{k}\right)\left(\frac{T}{t}\right)^{c\sqrt{D_1}}\left(\frac{1}{t}cY+ \frac{1}{t} \alpha Z\right)  \\
& = \frac{1}{n} exp\left(\gamma\frac{c^2D_1}{2}\right)\sum_{t=t_0+1}^T\left(\frac{T}{t}\right)^{\alpha D_2}\left(\frac{T}{t}\right)^{c\sqrt{D_1}}\left(\frac{1}{t}cY+ \frac{1}{t} \alpha Z\right) \\
& =  \frac{1}{n}(cY + \alpha Z) exp\left(\gamma\frac{c^2D_1}{2}\right)\sum_{t=t_0+1}^T\left(\frac{T}{t}\right)^{c\sqrt{D_1}+\alpha D_2}\frac{1}{t}\\
& =  \frac{1}{n}(cY + \alpha Z) exp\left(\gamma\frac{c^2D_1}{2}\right) T^{c\sqrt{D_1}+\alpha D_2}\sum_{t=t_0+1}^T t^{-c\sqrt{D_1}-\alpha D_2 - 1}\\
& \le \frac{1}{n}(cY + \alpha Z) exp\left(\gamma\frac{c^2D_1}{2}\right) T^{c\sqrt{D_1}+\alpha D_2} \int_{t_0}^{T}x^{-c\sqrt{D_1}-\alpha D_2 - 1}dx\\
& \le \frac{1}{n}(cY + \alpha Z) exp\left(\gamma\frac{c^2D_1}{2}\right)\frac{T^{c\sqrt{D_1}+\alpha D_2} }{c\sqrt{D_1}+\alpha D_2}\frac{1}{t_0^{c\sqrt{D_1}+\alpha D_2}}.
\end{align*}
\end{proof}

\section{Proof of Corollary~\ref{corollary2}} \label{appx:proof_corollary2}
\begin{proof}
By Lemma ~\ref{lem:stability}, we have that 
\begin{align*}
R(f(.,z))& := \mathbb{E}|f(x_T;z) - f(x'_T;z)| \\
& \le  2M\frac{t_0}{n} + \mu\Delta_T \\
& \le   2M\frac{t_0}{n} + \frac{1}{n}(cY + \alpha Z) 
\times
exp\left(\gamma\frac{c^2D_1}{2}\right)\frac{T^{c\sqrt{D_1}+\alpha D_2} }{c\sqrt{D_1}+\alpha D_2}\frac{\mu}{t_0^{c\sqrt{D_1}+\alpha D_2}} \\
\end{align*}
Denote 
\begin{align*}
A & := c\sqrt{D_1}+\alpha D_2\\
B&:= (cY + \alpha Z) \times exp\left(\gamma\frac{c^2D_1}{2}\right)\frac{\mu}{c\sqrt{D_1}+\alpha D_2}
\end{align*}
We then can rewrite
$$R(f(.,z)) \le \frac{1}{n}\left[2Mt_0 + B\left(\frac{T}{t_0}\right)^A\right]$$
We want to choose $t_0$ to minimize the right hand-side. We consider the function
$$g: x \mapsto 2Mx + B\left(\frac{T}{x}\right)^A$$
We have 
$$g'(x) = 2M - AB \frac{T^A}{x^{A+1}}$$
Thus the function is approximately minimize when 
$$t_0 = \left(\frac{AB}{2M}\right)^{\frac{1}{A+1}}T^{\frac{A}{A+1}}$$
Substitute this $t_0$ into the stability bound, we obtain
$$R(f(.,z)) \le \frac{1}{n}\left[2M\left(\frac{AB}{2M}\right)^{\frac{1}{A+1}} + \frac{B}{\left(\frac{AB}{2M}\right)^{\frac{A^2}{A+1}}}\right]T^{\frac{A}{A+1}}.$$
\end{proof}

\section{Supportive lemmas} 
\begin{lemma} \label{lemma4}
Let $f(.,z)$ be $\mu$-Lipschitz and $L$-smooth for all $z$. Assume that the assumption (3) is satisfied. For any $t\in \mathbb{N}$, we then have 
\begin{align*}
    K &:= \left\|\frac{1}{\sqrt{v_{t+1}+\epsilon}} - \frac{1}{\sqrt{v'_{t+1}+\epsilon}}\right\|_2  \le \frac{1}{2\sqrt{\lambda_1+\epsilon}(\lambda_1+\epsilon)^2}\|v_{t+1} - v'_{t+1}\|_2
\end{align*}
\end{lemma}
\begin{proof}
We have 
\begin{align*}
K & = \left[\sum_{j = 1}^d\left(\frac{1}{\sqrt{v_{t+1,j}+\epsilon}} - \frac{1}{\sqrt{v'_{t+1,j}+\epsilon}}\right)^2\right]^{1/2}\\
& = \left[\sum_{j = 1}^d\left(\frac{\sqrt{v_{t+1,j}+\epsilon} - \sqrt{v'_{t+1,j}+\epsilon}}{\sqrt{v_{t+1,j}+\epsilon}\sqrt{v'_{t+1,j}+\epsilon}}\right)^2\right]^{1/2}\\
& \le \left[\sum_{j=1}^d\frac{\left(\sqrt{v_{t+1,j}+\epsilon} - \sqrt{v'_{t+1,j}+\epsilon}\right)^2}{(\lambda_1+\epsilon)(\lambda_1+\epsilon)}\right]^{1/2} \\
& = \frac{1}{\lambda_1+\epsilon}\left[\sum_{i=1}^d\left(\sqrt{v_{t+1,j}+\epsilon} - \sqrt{v'_{t+1,j}+\epsilon}\right)^2\right]\\
& =  \frac{1}{\lambda_1+\epsilon} \left[\sum_{j = 1}^d\frac{(v_{t+1,j} - v'_{t+1, j})^2}{\left(\sqrt{v_{t+1,j}+\epsilon} + \sqrt{v'_{t+1,j}+\epsilon}\right)^2}\right]^{1/2}\\
& \le \frac{1}{2\sqrt{\lambda_1+\epsilon}(\lambda_1+\epsilon)}\|v_{t+1} - v'_{t+1}\|_2.
\end{align*}
\end{proof}

\section{Additional Training Parameters for Synthetic Data} \label{appx:synthetic}

\begin{figure*}[ht!]
\begin{subfigure}{.24\textwidth}
  \centering
  \includegraphics[width=0.89\linewidth]{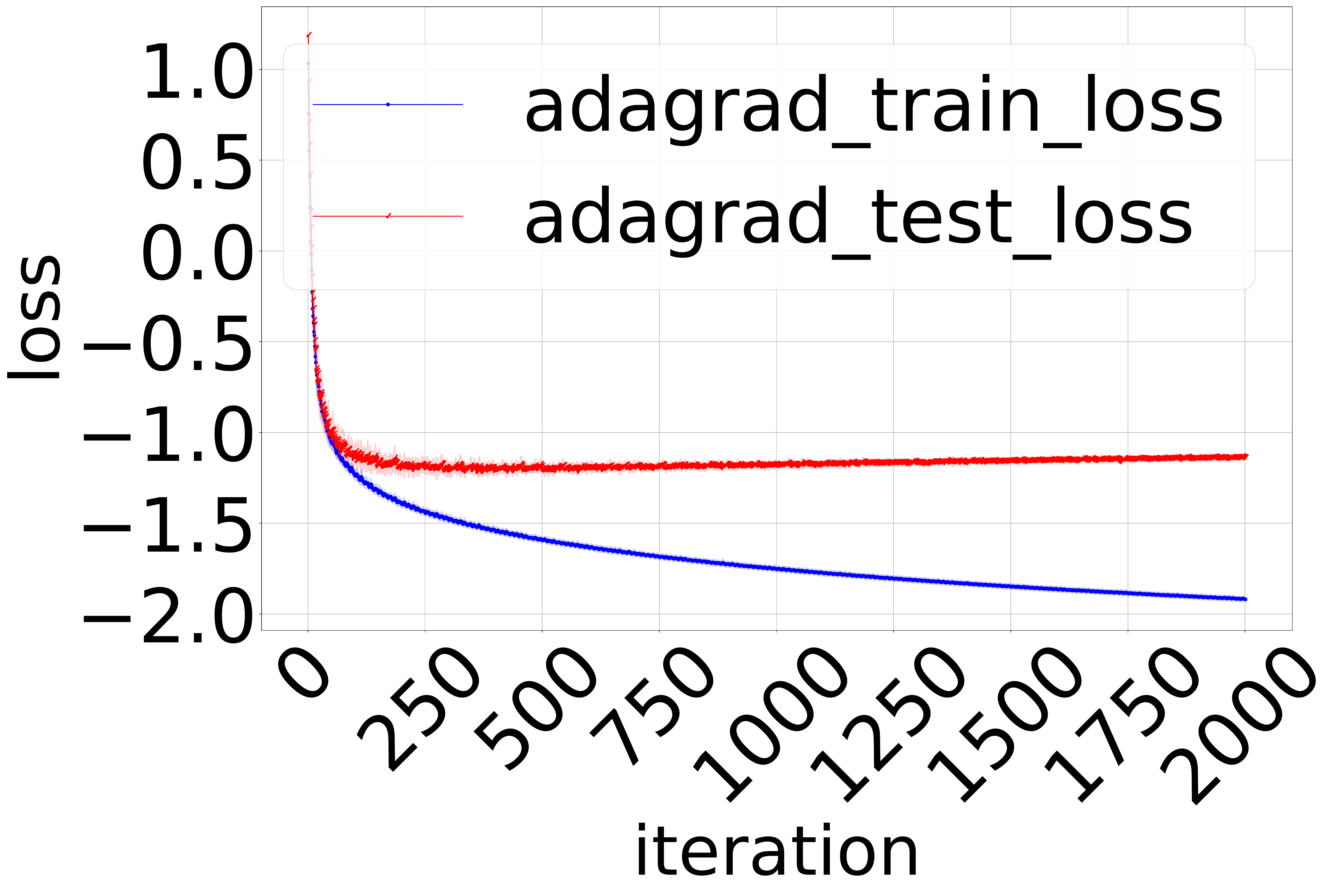}  
  \label{fig:sub-first}
\end{subfigure}
\begin{subfigure}{.24\textwidth}
  \centering
  \includegraphics[width=0.89\linewidth]{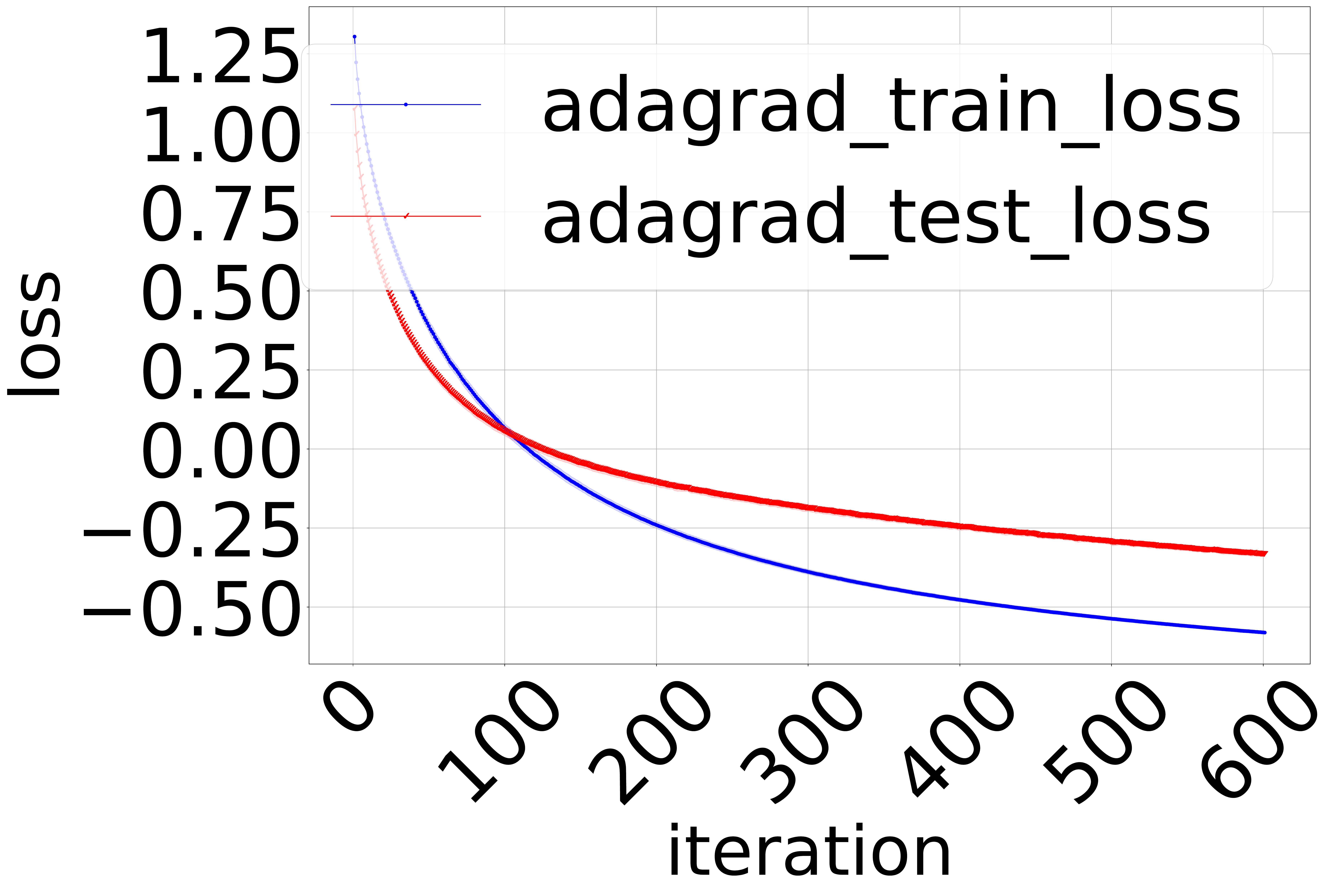}  
  \label{fig:sub-second}
\end{subfigure}
\begin{subfigure}{.24\textwidth}
  \centering
  \includegraphics[width=0.89\linewidth]{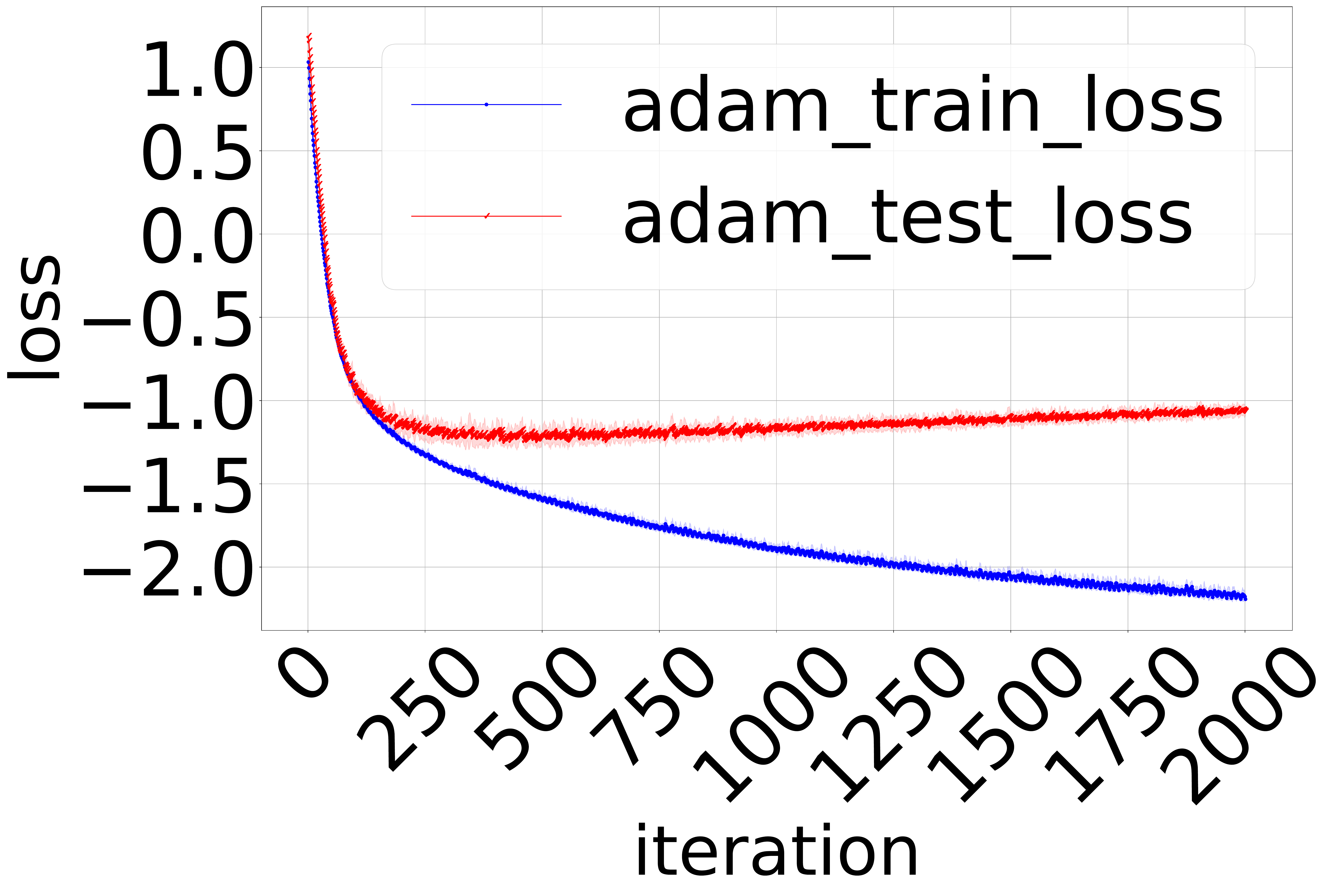}  
  \label{fig:sub-first}
\end{subfigure}
\begin{subfigure}{.24\textwidth}
  \centering
  \includegraphics[width=0.89\linewidth]{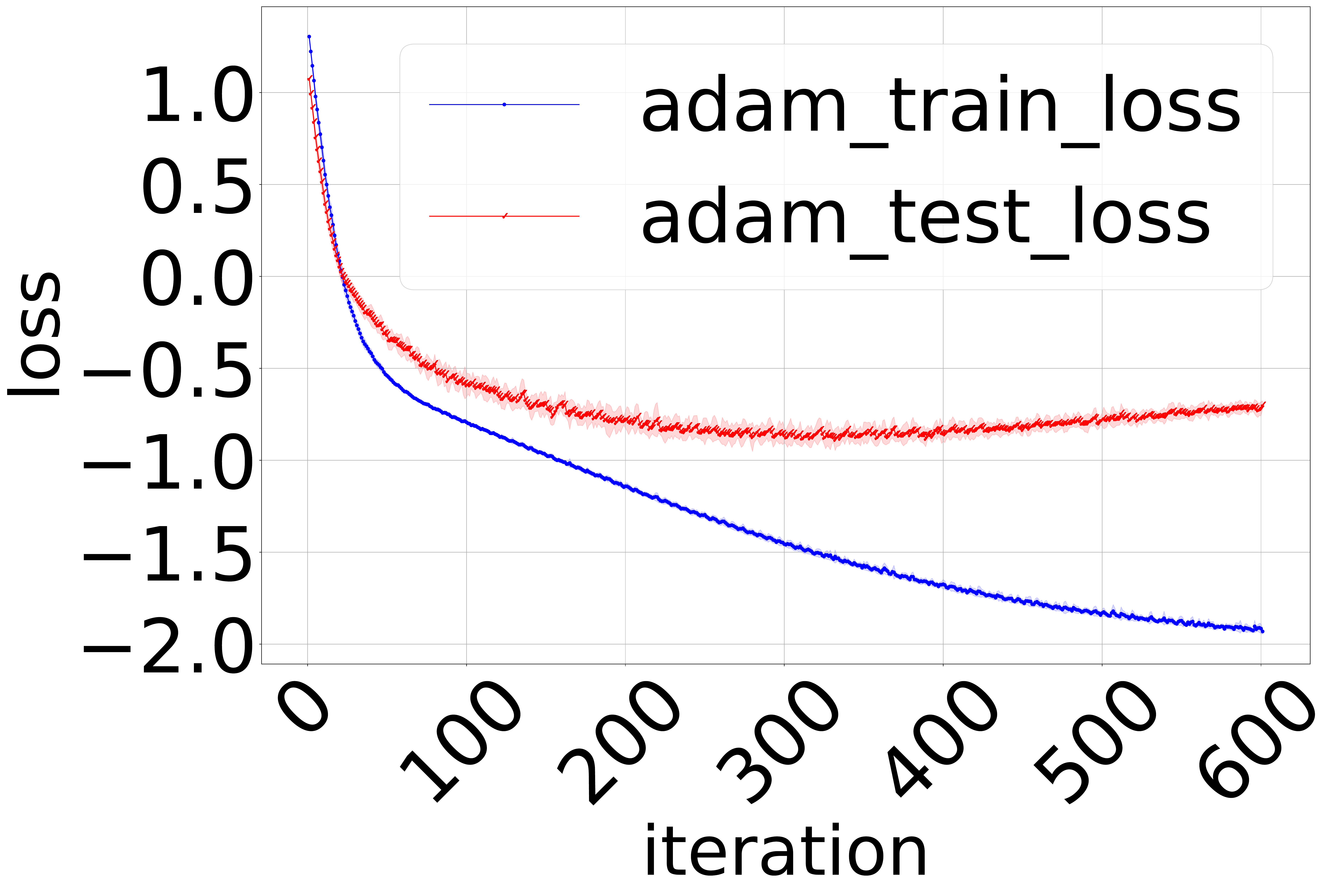}  
  \label{fig:sub-second}
\end{subfigure}
\caption{Train loss and test loss of CLS and REG tasks for Adagrad (left) and Adam (right).} 
\label{fig:adagrad_adam_loss}
\end{figure*}

For CLS training, we use a shallow neural network with one hidden fully-connected layer, which has 1024 neurons and an output of 3 classes. For the last layer, we use cross-entropy loss. 
For optimizser, we use Adagrad with learning rate  0.001. For comparison, we also run Adam with learning rate $0.0001$ and $\beta_1 = 0, \beta_2 = 0.999$. A batch size of 3 is used in this task.

Compared to CLS task, we only use 128 neurons for the hidden layer and MSE loss and batch size of 5 in REG task. In this case, we use Adam with learning rate 0.001 and $\beta_1 = 0, \beta_2 = 0.999$.

Finally, for each experiment, we run 20 trials and then calculate the mean and standard deviation before plotting the results.

\section{Additional Training Parameters for Real Data} \label{appx:real}

For the Cifar10 classification task, we train the model using 1 GPU with batch size 32. Similar to the synthetic setting, For each setting, we run 20 trials and for each trial, we run 60 epochs to calculate means and standard deviations. In terms of weight initialization for convolution and fully-connected layers of VGG11, we use He initialization method (``fan-out'' mode for convolution). For more details, please see the code attached to this supplemental material.

\end{document}